\documentclass[twocolumn]{autart}    % Enable this line and disable the 
\usepackage{graphicx}          % Include this line if your 
\usepackage[dvips]{epsfig}    % or this line, depending on which
\usepackage{cite}
\usepackage{amsmath,amssymb,amsfonts}
\usepackage{graphicx}
\usepackage{textcomp}
\usepackage{xcolor}
\usepackage{graphicx}
\usepackage[colorlinks=true, allcolors=blue]{hyperref}
\usepackage{subcaption}

\usepackage{algorithm,algorithmic}
\usepackage{diagbox}
\usepackage{threeparttable}
\usepackage{tabularx}
\usepackage{multirow}
\usepackage{bm}
\usepackage{comment}

\newtheorem{definition}{Definition}
\newtheorem{assumption}{Assumption}

\newtheorem{theorem}{Theorem}
\newtheorem{lemma}{Lemma}

\newtheorem{remark}{Remark}
\newtheorem{proof}{Proof}

\begin{document}

\begin{frontmatter}

\title{ImprovDML: Improved Trade-off in Private Byzantine-Resilient Distributed Machine Learning}

\thanks[Paestum]{Bing Liu, Chengcheng Zhao, Li Chai, and Peng Cheng are with the State Key Laboratory of Industrial Control Technology, Zhejiang University, Hangzhou, China (email: \{bing\_liu, chengchengzhao, chaili, lunarheart\}@zju.edu.cn)}
\thanks[Rome]{Yaonan Wang is with the College of Electrical and Information Engineering, Hunan University, Changsha, China (email: yaonan@hnu.edu.cn)}

\author[Paestum]{Bing Liu},    % Add the 
\author[Paestum]{Chengcheng Zhao}, 
\author[Paestum]{Li Chai}, 
\author[Paestum]{Peng Cheng}, 
\author[Rome]{Yaonan Wang}

                             % keyword list or with the 
                                          % help of the Automatica 
                                          % keyword wizard

\begin{abstract}                          % Abstract of not more than 200 words.
Jointly addressing Byzantine attacks and privacy leakage in distributed machine learning (DML) has become an important issue. A common strategy involves integrating Byzantine-resilient aggregation rules with differential privacy mechanisms. However, the incorporation of these techniques often results in a significant degradation in model accuracy. To address this issue, we propose a decentralized DML framework, named ImprovDML, that achieves high model accuracy while simultaneously ensuring privacy preservation and resilience to Byzantine attacks. The framework leverages a kind of resilient vector consensus algorithms that can compute a point within the normal (non-Byzantine) agents' convex hull for resilient aggregation at each iteration. Then, multivariate Gaussian noises are introduced to the gradients for privacy preservation. We provide convergence guarantees and derive asymptotic learning error bounds under non-convex settings, which are tighter than those reported in existing works. For the privacy analysis, we adopt the notion of concentrated geo-privacy, which quantifies privacy preservation based on the Euclidean distance between inputs. We demonstrate that it enables an improved trade-off between privacy preservation and model accuracy compared to differential privacy. Finally, numerical simulations validate our theoretical results.
\begin{comment}
Both theoretical analysis and numerical simulations demonstrate that our method achieves lower learning error while ensuring resilience to Byzantine faults and maintaining the same level of privacy. Furthermore, we show that our framework 
Thus, a key problem is how to preserve privacy while minimizing the loss in learning performance. To tackle this challenge, we adopt a class of resilient vector consensus algorithms as the aggregation rule and employ concentrated geo-differential privacy to quantify the level of privacy protection. Through both detailed theoretical analysis and numerical simulations, we demonstrate that our approach achieves a better trade-off between privacy, and accuracy compared to existing methods.

\end{comment}

\end{abstract}

\end{frontmatter}

%%%%%%%%%%%%%%%%%%%%%%%%%%%%%%%%%%%%%%%%%%%%%%%%%%%%%%%%%%%%%%%%%%%%%%%%%%%%%%%%
\section{INTRODUCTION}\label{sec:intro}
\par With the rapid proliferation of computing devices and the exponential growth of data in recent years, it has become increasingly common to train machine learning models locally on distributed devices, 
i.e., distributed machine learning (DML). DML  can be categorized into two paradigms: the server-worker model \cite{dean2012large,srivastava2015training}
and the decentralized peer-to-peer model \cite{lian2017can,sun2022decentralized} as shown in Fig. \ref{fig:twokinds}. This paper focuses on the fully decentralized model, where agents exchange information peer-to-peer to train their models collaboratively. The decentralized setting offers several advantages, such as being robust to a single point of failure, enhancing scalability, and adapting to heterogeneous environments. As a result, it has garnered significant attention from both academia and industry \cite{wu2023byzantine,li2022byzantine}. 
\begin{figure}[htbp]
  \begin{subfigure}{0.48\linewidth}
    \centerline{\includegraphics[width=1.0\linewidth]{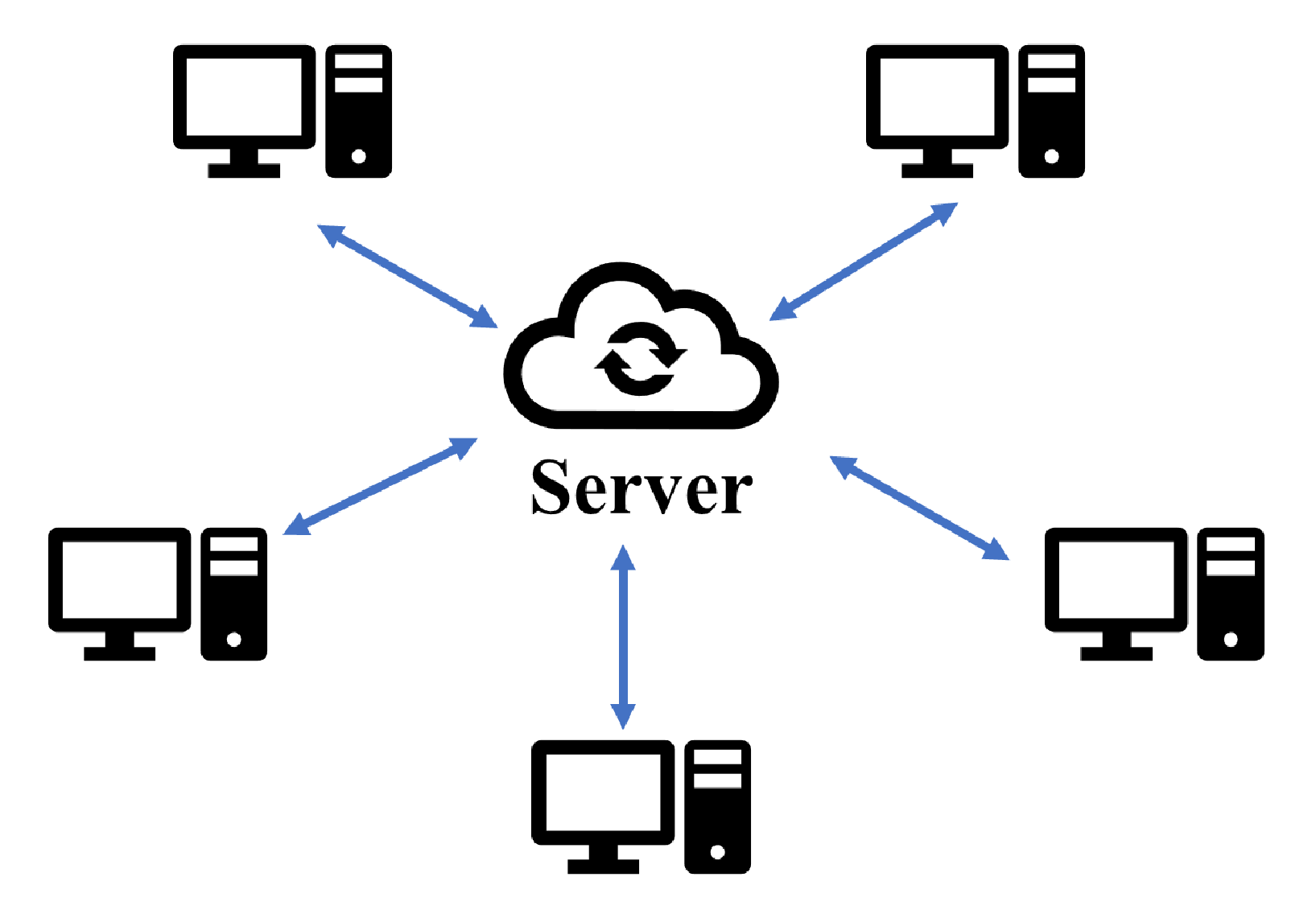}}
    \vspace{-5pt}
        \caption{Server-worker DML}  
    \label{fig:C_DML}  
  \end{subfigure}
  \hfill
  \begin{subfigure}{0.48\linewidth}
    \centerline{\includegraphics[width=1.0\linewidth]{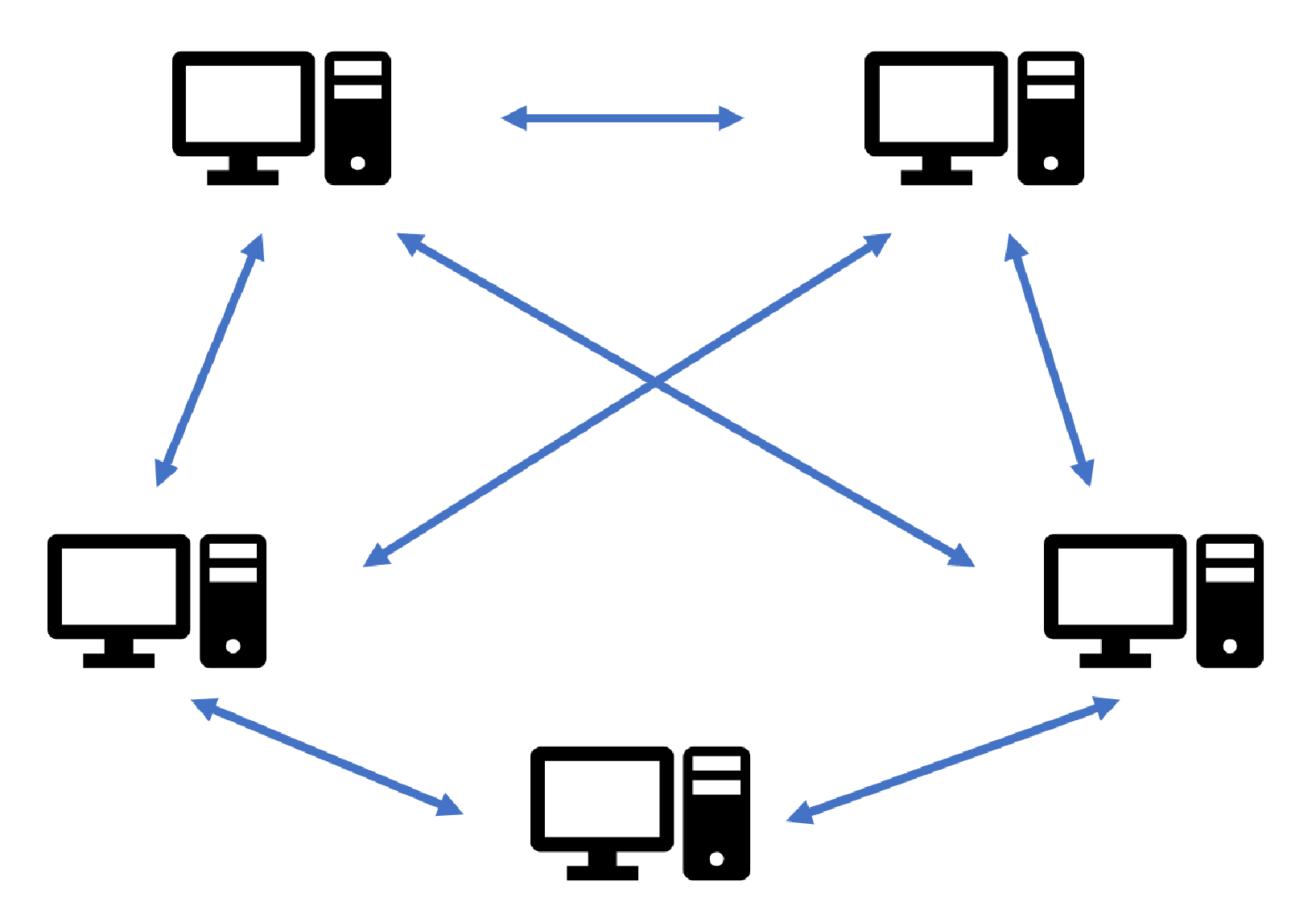}}
    \vspace{-5pt}
    \caption{Decentralized DML}
    \label{fig:D_DML}
  \end{subfigure}
  \vspace{-2pt}
  \caption{Two kinds of DML}
  \label{fig:twokinds}
\end{figure}
\par Although the decentralized setting offers significant advantages for machine learning, it also introduces critical vulnerabilities that can compromise the learning process. For instance, Byzantine attacks targeting DML can drastically degrade model accuracy \cite{baruch2019little,fang2020local}. In addition, privacy attacks such as gradient inversion attacks \cite{zhu2019deep,hatamizadeh2022gradvit} and model inversion attacks \cite{zhang2020secret,fredrikson2015model} pose serious privacy threats to DML systems. More importantly, these two types of attacks can be deployed simultaneously, further exacerbating the security risks of DML. For example, in medical data analytics, poisoning attacks can significantly reduce model accuracy, leading to misdiagnoses and posing serious threats to patient health. A straightforward approach to achieve both goals is to combine Byzantine-resilient aggregation rules \cite{he2022byzantine,wu2023byzantine,yang2024byzantine} with local differential privacy (local DP) \cite{dwork2014algorithmic,kasiviswanathan2011can}, as demonstrated by Ye \emph{et al.} \cite{ye2024tradeoff}. However, directly combining these two methods often leads to significant degradation in model accuracy. First, the estimations obtained through Byzantine-resilient aggregation rules inherently may lie outside the convex hull formed by the model parameters of normal agents—whereas in the absence of Byzantine agents, common averaging steps yield results within the convex hull. Secondly, local DP is often overly stringent, introducing substantial noise that further degrades model performance and leads to a worse trade-off between privacy and accuracy. Therefore, this paper ameliorates the aforementioned challenge from both perspectives.

\subsection{Related Work}
\par We first give a brief review of Byzantine-resilient server-worker DML, Byzantine-resilient decentralized DML, privacy-preserving decentralized DML, privacy-preserving Byzantine-resilient decentralized DML, and resilient vector consensus.
\par In \textbf{Byzantine-resilient server-worker DML}, there is a central parameter server. To achieve Byzantine resilience, the server performs robust aggregation to minimize the impact of faulty messages on the aggregation result. In existing works, the primary aggregation methods are ``majority-based" approaches, including coordinate-wise median \cite{yin2018byzantine}, geometric median \cite{chen2017distributed}, Krum \cite{blanchard2017machine}, Bulyan \cite{guerraoui2018hidden}, and FABA \cite{xia2019faba}. The core idea of these algorithms is to estimate the mean of the gradients of normal agents by filtering out values that deviate significantly from the mean. However, this estimation is not entirely accurate and contains some inherent error. These algorithms have been vulnerable to certain attacks \cite{baruch2019little}. Furthermore, extending them to decentralized DML can prevent agents from reaching consensus \cite{wu2023byzantine}. 
\par In \textbf{Byzantine-resilient decentralized DML}, agents communicate with each other directly to update their model parameters. Consequently, there is no parameter server to broadcast model parameters, and each agent must independently execute aggregation algorithms to achieve consensus among agents. In \cite{fang2022bridge, yang2019byrdie}, it is assumed that the number of Byzantine agents is either upper bounded or known. In this case, each normal agent discards the corresponding number of extreme values and computes the mean of the remaining values. Other notable works include a class of algorithms that enable each agent to set its own model parameter (or a weighted average of its neighbors' model parameters) as the baseline, and then discards a certain number of model parameters that are furthest from this baseline. These algorithms, including ClippedGossip \cite{he2022byzantine}, IOS \cite{wu2023byzantine}, and remove-then-clip \cite{yang2024byzantine}, produce an aggregated result that serves as an estimate of the convex combination of the normal agents' model parameters. However, this estimation has an upper-bounded error, which may cause the result to fall outside the normal agents' model parameters' convex hull, leading to a decrease in model accuracy. In \cite{fang2024byzantine}, Fang \emph{et al.} proposed BALANCE, in which each agent compares the model parameters received from its neighbors with its own and discards those that significantly differ in both direction and magnitude, by imposing an upper bound on the Euclidean norm of their difference. However, BALANCE does not guarantee the preservation of graph connectivity after the removal of suspicious model parameters, which may prevent normal agents from reaching convergence. In \cite{li2022byzantine}, Li \emph{et al.} extended the resilient vector consensus algorithm \cite{abbas2022resilient} to the Byzantine-resilient decentralized DML and analyzed the convergence. However, they simply extend the assumptions from single-agent machine learning \cite{bottou2018optimization} to the decentralized setting, imposing stricter requirements on the noise introduced by stochastic gradient descent (SGD) and the heterogeneity of local loss functions across agents. These assumptions are much stronger than those commonly adopted in decentralized DML, resulting in limited scalability.

\par In \textbf{privacy-preserving decentralized DML}, there are various methods to achieve privacy preservation, such as homomorphic encryption \cite{zhang2018admm,zhao2022pvd}, secure multi-party computation \cite{so2020scalable,lu2023privacy}, and differential privacy \cite{xu2021dp, wang2023tailoring}. Although the first two privacy-preserving methods achieve better training performance, their low computational efficiency has limited their widespread adoption in decentralized DML. In contrast, differential privacy has been widely applied due to its greater flexibility and efficiency. In \cite{xu2021dp}, Xu \emph{et al.} introduced Gaussian noise to the model parameters and analyzed both the convergence rate and the $(\varepsilon,\delta)$-DP guarantees of the algorithm in asynchronous decentralized DML. In \cite{wang2023tailoring}, Wang \emph{et al.} applied Laplace noise to the gradients and designed tailored gradient methods based on static-consensus and gradient-tracking, ensuring almost sure convergence to an optimal solution. Moreover, they demonstrated that the privacy budget remains bounded even as time approaches infinity. Furthermore, considering that adding noise in DP affects model accuracy, many studies have focused on privacy amplification \cite{beimel2014bounds}, aiming to achieve better privacy preservation under the same noise level. Some approaches exploit decentralization for privacy amplification, such as network DP and pairwise network DP proposed in \cite{cyffers2022privacy,cyffers2022muffliato}. The core idea is that each agent is unable to access all the information transmitted within the network; instead, it can only have access to the messages in which it is directly involved. In \cite{cyffers2024differentially}, a random walk mechanism was further introduced to enhance the effectiveness of pairwise network DP. However, these concepts cannot be applied to networks containing Byzantine agents, as they are omniscient and can access the global information of the network.
\par There are few works on the \textbf{privacy-preserving Byzantine-resilient decentralized DML}. In \cite{ghavamipour2024privacy}, Ghavamipour \emph{et al.} proposed a Byzantine-resilient aggregation rule based on cosine similarity and normalization, and employed secure multiparty computation to ensure privacy preservation. However, the computational and communication complexity of secure multiparty computation is excessively high, and certain nonlinear activation functions require polynomial approximations, which limits the broader adoption of this approach. In \cite{ye2024tradeoff}, Ye \emph{et al.} integrated $(\varepsilon,\delta)$-DP with existing resilient DML algorithms i.e., trimmed mean, ClippedGossip, and IOS, by injecting multivariate Gaussian noise. Additionally, they also provided an analysis of the trade-off between privacy and accuracy. However, directly applying DP requires larger noise, and the existing aggregation methods introduce estimation error, both of which simultaneously lead to a decline in learning accuracy and a sharp trade-off.

\par \textbf{Resilient vector consensus (RVC)} has similarities with decentralized DML. In RVC, normal agents achieve consensus by updating their states through local neighbor interactions while maintaining their states within the convex hull of initial states, despite Byzantine agents. In many studies on RVC, each normal agent is required to select a point that remains within the convex hull of the state vectors of its normal neighbors at every iteration \cite{abbas2022resilient,yan2022resilient,mendes2013multidimensional,park2017fault}. 
The algorithms in decentralized DML \cite{wu2023byzantine,he2022byzantine,yang2024byzantine} are designed to estimate the convex hull of normal neighbors' model parameters with certain errors. While RVC algorithms can directly compute an exact point within the convex hull (Fig. \ref{fig:convexhull-error} illustrates the detailed comparison.), why not explore the possibility of adapting RVC algorithms to decentralized DML? 
However, compared to RVC, decentralized DML presents additional analytical challenges. Due to the extra step of local SGD, the consensus analysis among agents' model parameters becomes more complex. Furthermore, since the goal of decentralized DML is to minimize the global loss function, it is also necessary to analyze the learning error induced by various factors.
\begin{figure}[htbp]
    \centering
    \includegraphics[width=0.5\linewidth]{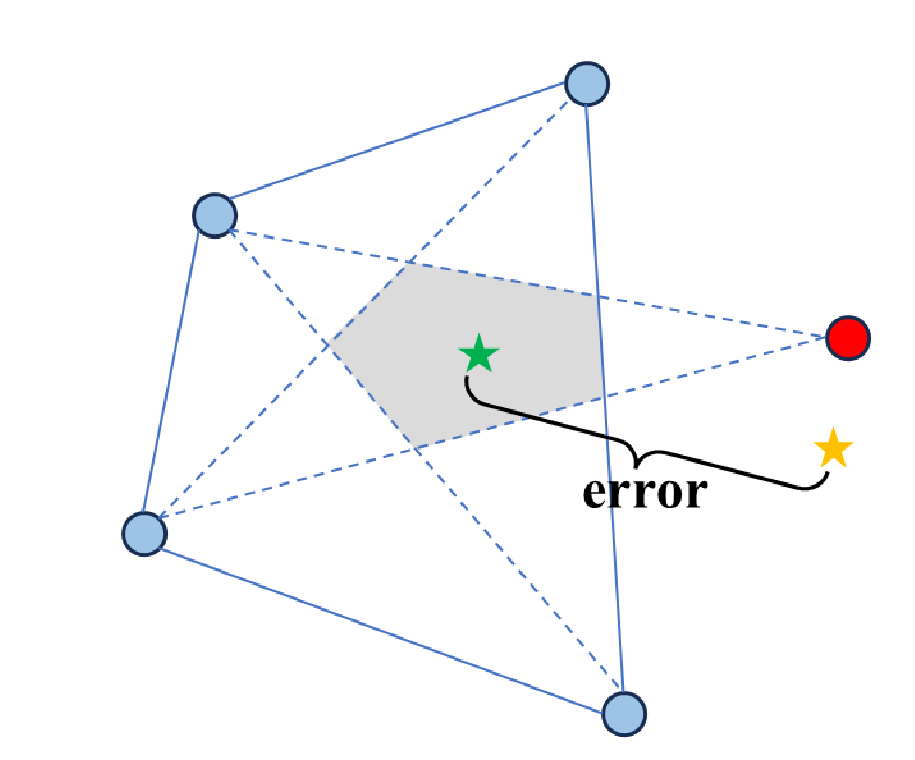}
    \caption{Comparison on aggregation rules. The blue points represent normal agents, while the red point denotes a Byzantine agent. The green star indicates the aggregation result computed by the RVC algorithms, and the yellow star denotes the result obtained from Byzantine-resilient decentralized DML algorithms \cite{he2022byzantine,wu2023byzantine,yang2024byzantine}. It can be observed that, due to the presence of errors, the result from \cite{he2022byzantine,wu2023byzantine,yang2024byzantine} may lie outside the convex hull formed by the normal agents.}
    \label{fig:convexhull-error}
\end{figure}

\subsection{Main Contributions}
 \par In this paper, we focus on decentralized DML where directly combining Byzantine-resilient aggregation rules with DP often results in significant learning accuracy degradation and an unbalanced privacy-accuracy trade-off. We address this issue by utilizing RVC algorithms as resilient aggregation rules to reduce learning error. In addition, we adopt a more precise and flexible privacy notion, namely concentrated geo-privacy (CGP) \cite{liang2023concentrated}, which incorporates input distance as a privacy parameter, to reach a better privacy-accuracy trade-off. The main contributions are summarized as follows.
\begin{itemize}
    \item  We propose ImprovDML, a decentralized DML framework that achieves both privacy preservation and resilience against Byzantine attacks, while removing the commonly imposed bounded gradients assumption in decentralized DML with DP.
    \item We analyze the consensus of model parameters among normal agents, as well as the asymptotic learning error of the final result under the proposed ImprovDML framework. Compared to existing algorithms, ImprovDML achieves a lower learning error.
    \item We derive a tight upper bound of CGP and analyze the corresponding trade-off between privacy and accuracy. Through a detailed comparison with DP, we demonstrate that CGP achieves an improved and more accurate trade-off. 
\end{itemize}

\par The organization of this work is as follows. We introduce problem formulation in Section \ref{sec:problem}. Then, we provide the main results in Section \ref{sec:mainresults}. Simulation results are presented in Section \ref{sec:sim}. Finally, conclusions and future work are outlined in Section \ref{sec:con}.

\par \textbf{Basic Notations:} Throughout this paper, we denote by $\mathbb{R}^{d}$ the $d$-dimensional Euclidean space, 
by $\mathbb{R}^{m\times d}$ the space of all $m\times d$-dimensional matrices. Let $||\cdot||,$ $||\cdot||_F$, and $||\cdot||_\mathrm{S}$ denote the $2$-norm, Frobenius norm, and spectral norm, respectively.
The natural logarithm is denoted by $\log$. Let $\mathbf{A}$ be a matrix and we denote by $\mathbf{A}^{\top}$ its transpose. We denote by $I_d$ the $d$-dimensional identity matrix. A matrix is row-stochastic if all its elements are nonnegative and each row sums up to $1$. For a given point set $C\subseteq\mathbb{R}^d$, 
we denote by $\text{conv}(C)$ the convex hull of the point set $C$, by $|C|$ its cardinality. We denote a matrix consisting of $n$ row vectors by $\mathbf{x}=[x_1,x_2,\ldots,x_n]^{\top}$, where $x_i$ denotes the $i$th vector. For two $n$-tuples of vectors 
 $\mathbf{x}=[x_1,x_2,\ldots,x_n]^\top$ and $\mathbf{x}'=[x_1',x_2',\ldots,x_n']^\top$, $||\mathbf{x}-\mathbf{x}'||$ stands for the maximum distance among all the vectors, which is $\max_i ||x_i-x_i'||$. Given a random variable $z$, we denote by $\mathbb{E}(z)$ the expectation of $z$. For a $d$-dimensional Gaussian distribution, its probability density function (pdf) is given by $\text{Gau}(\mu,\Sigma)=\frac{1}{\sqrt{(2\pi)^d|\Sigma|}}\mathrm{exp}\left(-\frac{1}{2}(x-\mu)^{\top}\Sigma^{-1}(x-\mu)\right)$.

\begin{definition}
($(\varepsilon,\delta)$-DP). Given spaces $U$, $V$, and $\varepsilon$, $\delta \in \mathbb{R} \geq 0$, a randomized function $\mathcal{M}:U\rightarrow V$ is $(\varepsilon,\delta)$-differentially private, if for any pair of inputs $x$, $x'\in U$ and any $S\subseteq V$, $\mathrm{P}\{\mathcal{M}(x)\in S\}\leq e^{\varepsilon}\mathrm{P}\{\mathcal{M}(x')\in S\}+\delta$ holds. 
\end{definition}
\begin{definition}
($\rho$-CGP). Given spaces $U$, $V$, and $\rho\in \mathbb{R}\geq 0$, a randomized function $\mathcal{M}:U\rightarrow V$ is said to satisfy $\rho$-CGP, if for any inputs $x,x'\in U$ and all $\alpha>1$, then it holds that $D_{\alpha}(\mathcal{M}(x)\|\mathcal{M}(x'))\leq\alpha\rho \cdot ||x-x'||^2$, where $D_\alpha$ represents the R\'enyi divergence between two distributions \cite{renyi1961measures}.
\end{definition}

\section{PROBLEM FORMULATION}\label{sec:problem}
\subsection{Network Model}
\par We consider a decentralized network represented as a directed graph $\mathcal{G}(k)=(\mathcal{V},\mathcal{E}(k))$, where $\mathcal{V}=\{1,2,\ldots,n\}$ denotes the set of agents and $\mathcal{E}(k) \subseteq \mathcal{V} \times \mathcal{V}$ represents the set of communication links. The set of in-neighboring agents of a given agent $i$ is denoted as $\mathcal{N}_i(k)=\{j\in \mathcal{V} \mid (j,i)\in \mathcal{E}(k)\}$ and its out-neighbors are $\{j\in \mathcal{V} \mid (i,j)\in \mathcal{E}(k)\}$.  The system has two kinds of agents, i.e., normal agents and Byzantine agents. Normal agents update their local model parameters through interactions with neighbors following the predefined rule. In contrast, Byzantine agents exhibit arbitrary and unpredictable behavior. A Byzantine agent can send distinct arbitrary messages to different neighbors \cite{leblanc2013resilient}. 
%, whereas a malicious agent is restricted to transmitting the same state to all its neighbors 
%\cite{leblanc2013resilient}. Since malicious agents form a subset of Byzantine agents, 
%In this paper, we focus on the worst-case scenario where all faulty agents behave as Byzantine agents.
We denote the set of Byzantine agents by $\mathcal{F} \subseteq \mathcal{V}$ and their total count as $f=|\mathcal{F}|$.  
The set of normal agents is given by $\overline{\mathcal{V}}=\mathcal{V} \setminus \mathcal{F}$, and its cardinality is $\overline{n}=n-f$. Without loss of generality, we suppose 
that the first $\overline{n}$ agents are normal, i.e., $\overline{\mathcal{V}}=\{1,2,\ldots,\overline{n}\}$. The network among normal agents is represented by the directed graph $\overline{\mathcal{G}}(k)=(\overline{\mathcal{V}},\overline{\mathcal{E}}(k))$, where $\overline{\mathcal{E}}(k)\subseteq{\overline{\mathcal{V}}} \times 
\overline{\mathcal{V}}$. For each normal agent $i \in \overline{\mathcal{V}}$, we denote its set of normal in-neighbors as $\overline{\mathcal{N}}_i(k)$ and the number of Byzantine in-neighbors in $\mathcal{G}(k)$ as $n_{f_i}(k)$. 
\par We consider the presence of privacy attackers both inside and outside the network, aiming to infer agents' local datasets by intercepting the information transmitted over the network. Internal attackers include honest-but-curious agents and Byzantine agents. The honest-but-curious attacker behaves like a normal agent by following the predefined algorithm but attempts to infer private information of other agents from the data received during network interactions. In contrast, Byzantine agents, who not only aim to disrupt the learning process but also seek to compromise the privacy of other agents. External attackers, such as eavesdroppers, are assumed to have knowledge of the network topology and can wiretap communication channels, thereby gaining access to the exchanged information without compromising the learning process.

\subsection{Decentralized DML and Algorithm Design}
\par Each normal agent $i$ maintains a local model parameter $x_i \in \mathbb{R}^d$ and an associated loss function $f_i(x_i; \xi_i)$, where $\xi_i$ denotes a mini-batch of data points randomly sampled from its local dataset $\mathcal{D}_i$. Specifically, $\xi_i = \{ \xi_{i,1}, \ldots, \xi_{i,b} \}$, with $b$ denoting the batch size. Therefore, $f_i(x_i; \xi_i)=\frac{1}{b}\sum_{j=1}^{b}f_i(x_i;\xi_{i,j})$. The expected loss function is denoted by $f_i(x)=\mathbb{E}_{\xi_i}[f_i(x;\xi_i)]$. The local dataset $\mathcal{D}_i$ may vary across agents, having different sizes $|\mathcal{D}_i|$. Then, under Byzantine attacks, the goal is to find the optimal vector $x^*$ that minimizes the average of the normal agents' local loss functions, denoted as $f(x)=\mathbb{E}_\xi[f(x,\xi)]$, where $\xi=\{\xi_1,\ldots,\xi_{\overline{n}}\}$, i.e.,
\begin{equation}
\label{eq:learning_problem}
\begin{aligned}
x^* \in \arg\min_{x\in \mathbb{R}^d} f(x)=\arg\min_{x\in \mathbb{R}^d} \frac{1}{\overline{n}}\sum\limits_{i\in \overline{\mathcal{V}}}\mathbb{E}[f_i(x;\xi_i)].
\end{aligned}
\end{equation}
\par To address problem \eqref{eq:learning_problem}, we propose Algorithm \ref{alg:PP-SGD-ADRC}, an SGD variant for decentralized DML that simultaneously ensures privacy preservation and resilience. In the initialization phase, each normal agent initializes its local parameter $x_i$ randomly. The algorithm then proceeds through three phases: the local SGD phase, the transmission phase, and the aggregation phase. In the local SGD phase, each normal agent samples its local dataset with a subsampling rate $0<\zeta_i<1$ and computes the corresponding stochastic gradient $\nabla f_i(x_i(k);\xi_i(k))$ based on the sampled data. During the update, noise is added to the gradient, yielding the update rule $\tilde{x}_i(k)=x_i(k)-\gamma[\nabla f_i(x_i(k);\xi_i(k))+\eta_i(k)]$. Subsequently, $\tilde{x}_i(k)$ is transmitted to all its out-neighbors. The final phase is the aggregation phase, where resilience is achieved. Since the messages received by each normal agent may contain Byzantine model parameters, a Byzantine-resilient aggregation rule is required. We choose a series of RVC algorithms as the aggregation rule, which is denoted as $\mathcal{R}(\cdot)$. The aggregation result is denoted as $s_i(k)$, and each normal agent updates its model parameter by computing a weighted average of its current parameter and the aggregation result $s_i(k)$.

\begin{algorithm}[htbp]
\footnotesize
\caption{PP-RSGD}\label{alg:PP-SGD-ADRC}
\begin{algorithmic}
\STATE {\textbf{Input:}} $\mathcal{G}(k)$, $\overline{\mathbf{x}}(k)$, $\Sigma=\sigma^2 I_d$, $\gamma$, $K$
        \STATE {\bf Output:} $x_i(K-1)$ 
        \STATE {\bf Initialization:} Randomly initialize the model parameters of normal agents $\overline{\mathbf{x}}_0$, $k=0$
        \FOR {$\{k=0,1,\ldots,K-1\}$}
\FOR{each normal agent $i\in \overline{\mathcal{V}}$} 
	\STATE {\bf Local SGD phase:}
	\STATE Compute $\nabla f_i(x_i(k);\xi_i(k))$.
        \STATE Add noise to the gradient and update the local parameter
 \begin{equation}\label{eq:add_noise}
 \tilde{x}_i(k)=x_i(k)-\gamma[\nabla f_i(x_i(k);\xi_i(k))+\eta_i(k)],     
 \end{equation}
where $\gamma>0$ represent the step size and $\eta_i(k)\in \mathbb{R}^d$ is a zero-mean decaying $d$-dimensional Gaussian noise with covariance matrix $\Sigma$.
	\STATE {\bf Transmission phase:}
        \STATE Transmit $\tilde{x}_i(k)$ to all its out-neighbors.
	\STATE {\bf Aggregation phase:}	
        \STATE Calculate the aggregation of its in-neighbors' parameters: 
        \begin{equation}
        s_i(k)=\mathcal{R}(\tilde{x}_j(k)), j\in \mathcal{N}_i(k).
        \end{equation}
	\STATE Update its parameter following:
	\begin{equation}\label{eq:update}
        x_i(k+1)=\beta_i(k)s_i(k)+(1-\beta_i(k))\tilde{x}_i(k),
        \end{equation} 
where $0<\beta_i(k)<1$.
	\ENDFOR
        \ENDFOR
\end{algorithmic}
\end{algorithm}

\subsection{RVC Algorithms}
\par In this section, we introduce the details of the RVC algorithms selected as the Byzantine-resilient aggregation rules. Specifically, these include the Byz-Iter algorithm from 
\cite{vaidya2014iterative}, the ADRC algorithm from \cite{park2017fault}, the resilient convex combination algorithm from \cite{wang2018resilient}, ADRC using centerpoint algorithm from \cite{abbas2022resilient}, and the resilient consensus Algorithm $1$ from \cite{yan2022resilient}. These algorithms share a common feature: in each iteration, every normal agent receives messages from its in-neighbors, including Byzantine agents, and computes a secure point that is guaranteed to lie within the convex hull formed by all normal in-neighbors. Therefore, we can derive the explicit formulations of this series of Byzantine-resilient aggregation algorithms,
\begin{equation}
s_i(k)=\mathcal{R}(\tilde{x}_j(k))=\sum a_{ij}(k)\tilde{x}_j(k), j\in \overline{\mathcal{N}}_i(k),
\end{equation}
where $a_{ij}(k)>0$ and $\sum_{j\in \overline{\mathcal{N}}_i(k)} a_{ij}(k)=1$.
\par Although many algorithms have been proposed to find a point within the intersection of multiple convex hulls, all these algorithms are fundamentally based on Helly's Theorem \cite{danzer1963helly}. Consequently, the theoretical condition ensuring the validity of the resulting convex combination is shared across all approaches. Specifically, for each normal agent $i \in \overline{\mathcal{V}}$, if the number of Byzantine agents in its in-neighborhood satisfies $n_{f_i}(k) < \frac{|\mathcal{N}_i(k)|}{d+1}$, it is guaranteed that the computed point always lies within the convex hull of the normal agents' transmitted parameters. The difference among these algorithms lies in the implementation methods: for example, \cite{vaidya2014iterative,park2017fault} employ Tverberg partitions, while \cite{wang2018resilient,yan2022resilient} directly compute the intersection of convex hulls. In contrast, \cite{abbas2022resilient} utilizes the concept of a centerpoint. Among these methods, the centerpoint-based approach is the most computationally efficient. As demonstrated in \cite{chan2004optimal}, the problem can be formulated as a linear program, making it applicable in arbitrary dimensions.

\par Note that RVC algorithms suffer from high computational complexity with the state dimension. For instance, the linear programming formulation used in the algorithm proposed by \cite{chan2004optimal} has a worst-case complexity of $O(n^{d-1})$, where $n$ denotes the number of points and $d$ is the dimension. This worst-case complexity arises from an enumeration-based method that exhaustively considers a $d-1$-dimensional arrangement. In practice, however, such worst-case scenarios are relatively rare. In most cases, the number of hyperplanes that need to be considered is significantly smaller than the theoretical worst-case bound. Moreover, more efficient solvers—such as those based on interior-point methods \cite{ye2011interior}—can be employed to solve the problem effectively. Nonetheless, the computational burden generally increases as the state dimension grows.

\section{Main Resutls}\label{sec:mainresults}
\par In this section, we first analyze the consensus error among the model parameters of all normal agents. We then establish the convergence of the final result and derive the asymptotic learning error. Finally, we analyze the corresponding privacy guarantee under CGP and compare it with that provided by $(\varepsilon,\delta)$-DP. We first present two lemmas that are crucial to the proof. 
\par Lemma $1$ is widely adopted in the analysis of decentralized DML \cite{lian2017can,xu2021dp,wu2023byzantine}, and Lemma $2$ is from \cite{abbas2022resilient}.
\begin{lemma}
\cite{lian2017can,xu2021dp,wu2023byzantine} Given three vectors $a,\text{ }b, \text{ and }c\in \mathbb{R}^d$, we have
\begin{equation}
||a+b+c||^2\leq \frac{1}{1-v}||a||^2+\frac{2}{v}||b||^2+\frac{2}{v}||c||^2,
\nonumber
\end{equation}
where $v\in (0,1)$.
\label{lm:three_vectors}
\end{lemma}

% \begin{proof}
% We first consider two vectors $||a+b||^2$:
% \begin{equation}
% \text{\footnotesize
% $\begin{aligned}
% ||a+b||^2 &= ||a||^2 + 2\langle a,b\rangle + ||b||^2 \\
% &= ||a||^2 + 2\langle ua,\tfrac{b}{u}\rangle + ||b||^2 \\
% &\leq (1+u^2)||a||^2 + \left(1+\tfrac{1}{u^2}\right)||b||^2
% \end{aligned}$}
% \end{equation}
% Let $v=\frac{u^2}{1+u^2}\in (0,1)$, we derive
% \begin{equation}
% \text{\footnotesize
% $\begin{aligned}
% ||a+b||^2 \leq \frac{||a||^2}{1-v} + \frac{||b||^2}{v}.
% \end{aligned}$}
% \end{equation}
% Hence, we can obtain the upper bound of $||a+b+c||^2$
% \begin{equation}
% \text{\footnotesize
% $\begin{aligned}
% ||a+b+c||^2&\leq \frac{1}{1-v}||a||^2+\frac{1}{v}||b+c||^2\\
% &\leq\frac{1}{1-v}||a||^2+\frac{1}{v}(\frac{1}{1-v'}||b||^2+\frac{1}{v'}||c||^2)\\
% &\overset{(v'=\frac{1}{2})}{=}\frac{1}{1-v}||a||^2+\frac{2}{v}||b||^2+\frac{2}{v}||c||^2.
% \end{aligned}$}
% \end{equation}
% \end{proof}

\par For convenience, let $\overline{\mathbf{x}}(k)=[x_1(k),x_2(k),\ldots,x_{\overline{n}}(k)]^\top \in \mathbb{R}^{\overline{n}\times d}$ denote the model parameters of the normal agents at iteration $k$, and let $\overline{\tilde{\mathbf{x}}}(k)=[\tilde{x}_1(k),\tilde{x}_2(k),\ldots,\tilde{x}_{\overline{n}}(k)]^\top \in \mathbb{R}^{\overline{n}\times d}$ be the model parameters of the normal agents before aggregation but after local SGD updates. We first represent the aggregation phase as a linear time-varying (LTV) system. 
\begin{lemma}\label{lm:ltv}
\cite{abbas2022resilient} If the condition $n_{f_i}(k) < \frac{|\mathcal{N}_i(k)|}{d+1}$ holds for any normal agent $i\in \overline{\mathcal{V}}$, the aggregation phase of Algorithm \ref{alg:PP-SGD-ADRC} can be represented as the following LTV system,
\begin{equation}
\text{\footnotesize
$\overline{\mathbf{x}}(k+1)=\mathbf{M}(k)\overline{\tilde{\mathbf{x}}}(k), k=0,1,2,\ldots,$}
\nonumber
\end{equation}
where $\mathbf{M}(k)$ is a row-stochastic matrix, and its $(i,j)$th entry is positive if and only if $j \in \overline{\mathcal{N}}_i(k)\cup\{i\}$.
\end{lemma}
\par Before presenting the detailed analysis, we introduce several common assumptions.
\begin{assumption}
(Lower bound). The loss function $f(x)$ is lower bounded by $f^*$; i.e., $f(x)\geq f^*,\forall x.$
\end{assumption}
\begin{assumption}
\label{ass:Lipschitz-continuous gradients}
(Lipschitz-continuous gradients). For each normal agent $i \in \overline{\mathcal{V}}$, the function $f_i(x; \xi_i)$ is differentiable, and its gradient $\nabla f_i$ is $L$-Lipschitz continuous with respect to $x$, i.e., $
\|\nabla f_i(x;\xi_i) - \nabla f_i(x';\xi_i)\| \leq L\|x - x'\|, \forall x, x' \in \mathbb{R}^d$.
%\begin{equation}
%||\nabla f_i(x,\xi_i)-\nabla f_i(y,\xi_i)||\leq %L||x-y||, \forall x,y.
%\nonumber and $||\nabla f_i(x,\xi_i)-\nabla f_i(x,\xi_i')||\leq L||\xi_i-\xi_i'||, \forall \xi_i,\xi_i'$.
%\end{equation}
\end{assumption}
\begin{assumption}
(Bounded variance). Given any $i \in \overline{\mathcal{V}}$ and $x$, the variance of the gradient $\nabla f_i(x; \xi_i)$ is upper-bounded, i.e., there exists $\theta>0$ so that $\mathbb{E}_{\xi_i\sim\mathcal{D}_i}[||\nabla f_i(x;\xi_i)-\nabla f_i(x)||^2]\leq \theta^2$.
In addition, for any $x$, the variance between the gradient of its expected loss function and the overall loss function is also upper bounded, i.e., there exists $\tau>0$ so that $\mathbb{E}_{i\sim\overline{\mathcal{V}}}[||\nabla f_i(x)-\nabla f(x)||^2]\leq \tau^2$.
%\begin{equation}
%\mathbb{E}_{i\sim\overline{\mathcal{V}}}[||\nabla %f_i(x)-\nabla f(x)||^2]\leq \tau^2.
%\end{equation}
\end{assumption}

\begin{assumption}
(Independent sampling). The data $\xi_i(k)$ are independently sampled across both iterations $k=0,1,\ldots$ and normal agents $i\in \overline{\mathcal{V}}$.
\end{assumption}
\begin{assumption}
\label{ass:network_connectivity}
(Connected graph). The graph of all normal agents, $\overline{\mathcal{G}}(k) = (\overline{\mathcal{V}}, \overline{\mathcal{E}}(k))$, is strongly connected, meaning that for any pair of agents $i, j \in \overline{\mathcal{V}}$, there exists a path from $i$ to $j$ and from $j$ to $i$. Furthermore, we assume that $||(I_d - \frac{1}{\overline{n}} \mathbf{1} \mathbf{1}^\top) \mathbf{M}(k) ||_{\mathrm{S}}^2\leq\mu<1$ \cite{wu2023byzantine}.
\end{assumption}
\begin{assumption}
\label{ass:limit_faulty}
(Limited number of Byzantine agents). For each normal agent $i\in \overline{\mathcal{V}}$ and any $k=0,1,2,\ldots$, the number of Byzantine agents in a normal agent's in-neighborhood should satisfy $n_{f_i}(k)< \frac{|\mathcal{N}_i(k)|}{d+1}$.
\end{assumption}
\begin{assumption}\label{ass:lipstichz-input}
(Lipschitz-continuous on inputs). The gradient $\nabla f_i(x_i;\xi_i)$ is $L'$-Lipschitz continuous with respect to $\xi_i$, i.e., $||\nabla f_i(x_i;\xi_i)-\nabla f_i(x_i;\xi_i')||\leq L'||\xi_i-\xi'_i||, \forall \xi_i,\xi'_i $.
\end{assumption}
\begin{comment}
\begin{assumption}
\label{ass:strongly_convex}
(Strongly convex loss function). The loss function $f(x)$ is $\mu$-strongly convex with $\mu>0$, i.e., $f(x)+\langle\nabla f(x),x'-x\rangle+\frac{\mu}{2}||x'-x||^2\leq f(x'),\forall x.x'\in \mathbb{R}^d.$
\end{assumption}
\end{comment}

\begin{remark}
Assumptions $1-5$ are quite common in decentralized DML \cite{wu2023byzantine,li2022byzantine,yang2024byzantine,fang2024byzantine}. In fact, in privacy-preserving DML, it is commonly assumed that the gradients of all normal agents' loss functions are uniformly bounded, with a known upper bound, in order to derive the gradient sensitivity \cite{xu2021dp,ye2024tradeoff,wang2019subsampled}. 
In practice, commonly used loss functions—such as mean squared error, cross-entropy, and Kullback-Leibler divergence—typically lack gradient extrema, and real-world data distributions often exhibit significant variability across their domains. As a result, establishing a strict upper bound on the gradient norm is challenging. To address this, gradient clipping is often employed to artificially enforce such a bound. However, gradient clipping may lead to degraded learning accuracy and introduces complex theoretical challenges \cite{koloskova2023revisiting}. In contrast,  using CGP alleviates the need to assume bounded gradients, and instead relies on Assumption \ref{ass:lipstichz-input}.

\end{remark}

\begin{comment}
\begin{proof}
We consider a normal agent $i\in \overline{\mathcal{V}}$, the aggregation $s_i(k)$ can be written as a nonzero convex combination
\begin{equation}
\text{\footnotesize
$s_i(k)=\sum_{j\in \overline{\mathcal{N}}_i(t)}a_{ij}(k)\tilde{x}_j(k),$}
\nonumber
\end{equation}
where $\sum_{j\in \overline{\mathcal{N}}_i(t)}a_{ij}(k)=1$. Therefore, we obtain 
\begin{equation}
\text{\footnotesize
$\begin{aligned}
x_i(k+1)&=\beta_i(k)\sum_{j\in \overline{\mathcal{N}}_i(t)}a_{ij}(t)x_j(k)+(1-\beta_i(k))\tilde{x}_i(k).
\end{aligned}$}
\nonumber
\end{equation}
Then, a linear time-varying system can be derived as
\begin{equation}
\text{\footnotesize
$\overline{\mathbf{x}}(k+1)=\mathbf{M}(k)\overline{\mathbf{x}}(k),\ k=0,1,2,\ldots,$}
\nonumber
\end{equation}
where $\overline{\mathbf{x}}(k)=[x_1(k),\ldots,x_{\overline{n}}(k)]^\top$ and
\begin{equation}
\text{\footnotesize
$[\mathbf{M}(k)]_{ij}=\begin{cases}
1-\beta_i(k) & \text{if } i=j\\
\beta_i(k)a_{ij}(k) & \text{if } i \neq j \text{ and } j\in \overline{\mathcal{N}}_i(k)\\
0  & \text{otherwise,}
\end{cases}$}
\end{equation}
which is a row-stochastic matrix.
\end{proof}
\end{comment}

\subsection{Consensus Analysis}
%We first prove that under Algorithm \ref{alg:PP-SGD-ADRC}, the parameters of each normal agent will converge to consensus in expectation. 
We define the consensus error as 
\begin{equation}
\text{\footnotesize
$\Delta(k)=\frac{1}{\overline{n}}\sum\limits_{i\in \overline{\mathcal{V}}}||x_i(k)-\overline{x}(k)||^2,$}
\nonumber
\end{equation}
where $\overline{x}(k)=\frac{1}{\overline{n}}\sum_{i\in \overline{\mathcal{V}}}x_i(k)$ is the average of the normal agents' model parameters at iteration $k$. 
\par Based on Assumption \ref{ass:network_connectivity} and from \cite{wu2023byzantine}, we define $\lambda(k) \triangleq 1 - | |(I_d - \frac{1}{\overline{n}} \mathbf{1} \mathbf{1}^\top) \mathbf{M}(k) ||_{\mathrm{S}}^2>0$. Let $\lambda\geq1-\mu>0$ denote the minimum value of $\lambda(k)$ across all $K$ iterations.
\begin{theorem}\label{th:consensus}
(Consensus). With Assumptions $2$-$6$, if the step size $\gamma$ of local SGD satisfies $\gamma < \frac{1}{2L}\sqrt{\frac{\lambda}{6(1-\lambda)(2-\lambda)}}$, the consensus error is bounded as  
\begin{equation}
\text{\footnotesize
$\mathbb{E}[\Delta(k)]\leq \Lambda^k\Delta(0)+\frac{2\gamma^2(8\theta^2+6\tau^2+2d\sigma^2)}{\lambda(1-\Lambda)},$}
\nonumber
\end{equation}
where $\Lambda=\frac{(1 - \lambda)[2\lambda + 24\gamma^2L^2(2 - \lambda)]}{(2 - \lambda)\lambda}$, with $0<\Lambda<1$. %It can be shown that, in the expected sense, the model parameters of all normal agents will 
%converge to consensus.
%\par Furthermore, we can derive the limited upper-bound cumulative non-consensus error starting from $k=1$ to $K$, which is 
\begin{comment}
\begin{equation}
\text{\footnotesize
$\sum_{k=0}^{K-1} \mathbb{E}[\Delta(k)]\leq \frac{1}{\lambda^2}\Delta(0) +O\bigg(\frac{\lambda'\overline{n}(8\theta^2+6\tau^2+2d\sigma^2)}{\lambda^2}\bigg).$}
\end{equation}
\end{comment}
\end{theorem}
\begin{proof}
By Lemma \ref{lm:ltv}, We can obtain 
\begin{equation}
\label{eq:th1.1}
\text{\footnotesize
$\begin{aligned}
&\quad \sum\limits_{i\in \overline{\mathcal{V}}}||x_i(k+1)-\overline{x}(k+1)||^2\\
&=||(I_d-\frac{1}{\overline{n}}\mathbf{1}\mathbf{1}^\top)\overline{\mathbf{x}}(k+1)||^2_F\\
&=||(I_d-\frac{1}{\overline{n}}\mathbf{1}\mathbf{1}^\top)\mathbf{M}(k)\overline{\tilde{\mathbf{x}}}(k)||^2_F\\
&=||(I_d-\frac{1}{\overline{n}}\mathbf{1}\mathbf{1}^\top)\mathbf{M}(k)(I_d-\frac{1}{\overline{n}}\mathbf{1}\mathbf{1}^\top)\overline{\tilde{\mathbf{x}}}(k)||^2_F\\
&\leq(1-\lambda)||(I_d-\frac{1}{\overline{n}}\mathbf{1}\mathbf{1}^\top)\overline{\tilde{\mathbf{x}}}(k)||^2_F.
\end{aligned}$}
\end{equation}
The final term in \eqref{eq:th1.1} can be reformulated as 
\begin{equation}
\text{\footnotesize
$\begin{aligned}
&\mathbb{E}[||(I_d-\frac{1}{\overline{n}}\mathbf{1}\mathbf{1}^\top)\overline{\tilde{\mathbf{x}}}(k)||^2_F]=\sum\limits_{i\in \overline{\mathcal{V}}}\mathbb{E}[||\tilde{x}_i(k)-\overline{\tilde{x}}(k)||^2].
\end{aligned}$}
\nonumber
\end{equation}
Considering the update function, we can obtain
\begin{equation}
\text{\footnotesize
$\tilde{x}_i(k)=x_i(k)-\gamma [\nabla f_i(x_i(k);\xi_i(k))+\eta_i(k)],$}
\nonumber
\end{equation}
and
\begin{equation}
\text{\footnotesize
$\begin{aligned}
\overline{\tilde{x}}(k)&=\overline{x}(k)-\frac{\gamma}{\overline{n}}\sum_{j\in \overline{\mathcal{V}}}\nabla f_j(x_j(k);\xi_j(k))-\frac{\gamma}{\overline{n}}\sum_{j\in \overline{\mathcal{V}}} \eta_j(k). 
\end{aligned}$}
\nonumber
\end{equation}
Therefore, we can obtain
\begin{equation}
\label{eq:consensus_1}
\text{\footnotesize
$\begin{aligned}
&\quad ||\tilde{x}_i(k)-\overline{\tilde{x}}(k)||^2\\
&=||x_i(k)-\overline{x}(k)+\gamma\bigg[\frac{\sum\limits_{j\in \overline{\mathcal{V}}}\nabla f_j(x_j(k);\xi_j(k))}{\overline{n}}\\
&\quad- \nabla f_i(x_i(k);\xi_i(k))\bigg]-\gamma\eta_i(k)+\frac{\gamma}{\overline{n}}\sum_{j\in \overline{\mathcal{V}}} \eta_j(k)||^2.
\end{aligned}$}
\end{equation}
Following Lemma \ref{lm:three_vectors}, it holds
\begin{equation}
\label{eq:18}
\text{\footnotesize
$\begin{aligned}
&\quad \mathbb{E}[||\tilde{x}_i(k)-\overline{\tilde{x}}(k)||^2]\\
&\leq \frac{\mathbb{E}[||[x_i(k)-\overline{x}(k)]||^2]}{1-v}+\frac{2\gamma^2\mathbb{E}[||\eta_i(k)-\frac{1}{\overline{n}}\sum\limits_{j\in \overline{\mathcal{V}}} \eta_j(k)||^2]}{v}\\
&\quad +\frac{2\gamma^2}{v}\mathbb{E}[||\nabla f_i(x_i(k);\xi_i(k))-\frac{\sum\limits_{j\in \overline{\mathcal{V}}}\nabla f_j(x_j(k);\xi_j(k))}{\overline{n}} ||^2].
\end{aligned}$}
\end{equation}
First, we obtain the upper bound of the noise term
\begin{equation}
\label{eq:19}
\text{\footnotesize
$\begin{aligned}
&\quad \mathbb{E}[||\eta_i(k)-\frac{1}{\overline{n}}\sum\limits_{j\in \overline{\mathcal{V}}} \eta_j(k)||^2]\\
&=\frac{1}{\overline{n}^2}\mathbb{E}[||\sum_{j\in \overline{\mathcal{V}}}(\eta_i(k)-\eta_j(k))||^2]\\
&=\frac{1}{\overline{n}^2}\{\mathbb{E}[||(\overline{n}-1)\cdot \eta_i(k)||^2]+\sum\limits_{j
\in \overline{\mathcal{V}}\backslash \{i\}}\mathbb{E}[||\eta_j(k)||^2]\}\\
&=\frac{(\overline{n}-1)^2+(\overline{n}-1)}{\overline{n}^2}\cdot d\cdot \sigma^2\leq d\sigma^2.
\end{aligned}$}
\end{equation}

\par Then, for the third term, according to Lemma 2 in \cite{wu2023byzantine}, we have 
\begin{equation}\label{eq20}
\text{\footnotesize
$\begin{aligned}
&\mathbb{E}[||\nabla f_i(x_i(k);\xi_i(k))-\frac{\sum\limits_{j\in \overline{\mathcal{V}}}\nabla f_j(x_j(k);\xi_j(k))}{\overline{n}}||^2]\\
\leq &6L^2 \mathbb{E}[||x_i(k)-\overline{x}(k)||^2]+3\tau^2+4\theta^2.       
\end{aligned}$}
\end{equation}
\par Substituting inequalities \eqref{eq:18}, \eqref{eq:19}, and \eqref{eq20} into equation \eqref{eq:consensus_1}, we obtain
\begin{equation}\label{eq:gr_compress}
\text{\footnotesize
$\begin{aligned}
&\sum\limits_{i\in \overline{\mathcal{V}}}\mathbb{E}[||\tilde{x}_i(k)-\overline{\tilde{x}}(k)||^2]\\
\leq&\bigg(\frac{1}{1-v}+\frac{12\gamma^2L^2}{v}\bigg)\sum\limits_{i\in \overline{\mathcal{V}}}\mathbb{E}[||x_i(k)-\overline{x}(k)||^2]\\
&\quad +\frac{8\gamma^2\theta^2+6\gamma^2\tau^2+2d \gamma^2\sigma^2}{v}.
\end{aligned}$}
\nonumber
\end{equation}
Then, we can obtain the induction 
\begin{equation}
\text{\footnotesize
$\begin{aligned}
\mathbb{E}[\Delta(k+1)]&\leq (1-\lambda)\bigg(\frac{1}{1-v}+\frac{12\gamma^2L^2}{v}\bigg)\mathbb{E}[\Delta(k)]\\
&\quad+(1-\lambda)\frac{8\gamma^2\theta^2+6\gamma^2\tau^2+2d \gamma^2\sigma^2}{v}.
\end{aligned}$}
\nonumber
\end{equation}
Let $v = \frac{\lambda}{2}$ and define $\Lambda = (1 - \lambda)\frac{2\lambda + 24\gamma^2L^2(2 - \lambda)}{(2 - \lambda)\lambda}$, we can derive
\begin{equation}
\text{\footnotesize
$\begin{aligned}
\mathbb{E}[\Delta(k+1)]&\leq \Lambda\mathbb{E}[\Delta(k)]+2\gamma^2\frac{1}{\lambda}(8\theta^2+6\tau^2+2d \sigma^2).
\end{aligned}$}
\nonumber
\end{equation}
By choosing  $\gamma < \frac{1}{2L} \sqrt{ \frac{\lambda}{6(1 - \lambda)(2 - \lambda)} }$,
we can ensure that $0 < \Lambda < 1$, which implies 
\begin{equation}
\text{\footnotesize
$\begin{aligned}
\mathbb{E}[\Delta(k)]&\leq \Lambda^k\Delta(0)+\sum_{k=0}^{k}\Lambda^k\lambda'\gamma^2(8\theta^2+6\tau^2+2d \sigma^2)\\
&\leq \Lambda^k\Delta(0) + \frac{2\gamma^2(8\theta^2+6\tau^2+2d \sigma^2)}{\lambda(1-\Lambda)}.
\end{aligned}$}
\nonumber
\end{equation}
\begin{comment}
Given constant step size as $\gamma=O\bigg(\sqrt{\frac{\overline{n}}{K}}\bigg)$, we have 
\begin{equation}
\text{\footnotesize
$\mathbb{E}[\Delta(k)]\leq(1-\lambda^2)^k\Delta(0)+ O\bigg(\frac{\lambda'\overline{n}(8\theta^2+6\tau^2+2d\sigma^2)}{\lambda^2K}\bigg).$}
\nonumber
\end{equation}
The corresponding cumulative non-consensus error is 
\begin{equation}     
\text{\footnotesize
$\begin{aligned}
\sum_{k=0}^{K-1} \mathbb{E}\Delta(k)\leq\frac{1}{\lambda^2} \Delta(0) + O\bigg(\frac{\lambda'\overline{n}(8\theta^2+6\tau^2+2d\sigma^2)}{\lambda^2}\bigg).
\end{aligned}$}
\nonumber
\end{equation}
\end{comment}
\end{proof}
\begin{remark}
It is worth noting that the final result does not explicitly include any term related to resilience. However, resilience is in fact implicitly embedded in the term $\lambda$. Since $\mathbf{M}(k)$ is a time-varying row-stochastic matrix with unknown specific elements, it is difficult to determine the exact influence of this term. This limitation is also encountered in other existing works \cite{wu2023byzantine, yang2024byzantine, he2022byzantine}. The advantage of our result in terms of resilience lies in its ability to achieve an accurate estimation of a point within the convex hull of the normal agents. In contrast, in \cite{wu2023byzantine, yang2024byzantine, he2022byzantine}, the presence of estimation errors introduces a perturbation to $\Lambda$, leading to its increase and consequently exacerbating the degree of non-consensus.
\end{remark}
We observe that, due to differences in the loss functions and datasets across normal agents, as well as the added noise, the model parameters of normal agents exhibit a bounded discrepancy that persists even after a large number of iterations. If we set $\Delta(0) = 0$, i.e., all normal agents share the same initial model parameter, and remove the added noise, we obtain
\begin{equation}
\mathbb{E}[\Delta(k)]\leq O\bigg(\gamma^2(\theta^2+\tau^2)\bigg),
\nonumber
\end{equation}
which aligns with the case that does not consider resilience and privacy \cite{koloskova2020unified}.
\subsection{Convergence Analysis}
In this subsection, we present the convergence result of Algorithm \ref{alg:PP-SGD-ADRC}.
\begin{theorem}
\label{th:convergence}
(Convergence). With Assumptions $1$-$6$, if the step size $\gamma$ of local SGD satisfies $\gamma \leq \frac{1}{2L}$, the convergence rate Algorithm \ref{alg:PP-SGD-ADRC} is given as follows
\begin{equation}\label{eq:convergence}
\text{\footnotesize
$\begin{aligned}
&\quad \frac{1}{K}\sum_{k=0}^{K-1}\mathbb{E}[||\nabla f(\overline{x}(k))||^2]\\
&\leq \frac{2\mathbb{E}[f(\overline{x}(0))-f(x^*)]}{\gamma K}+ \frac{12}{\overline{n}}\chi^2(4\theta^2+3\tau^2)+\frac{2\gamma\theta^2L}{\overline{n}}\\
&\quad +\frac{3(\frac{L^2}{\overline{n}^2}+\frac{2\chi^2}{\gamma^2}+24\chi^2L^2)}{K}\sum_{k=0}^{K-1} \mathbb{E}[\Delta(k)]+3(\frac{d}{\overline{n}}+4\chi^2d)\sigma^2.
\end{aligned}$}
\nonumber
\end{equation}
If we set $\gamma = c \cdot \sqrt{\frac{1}{K}}$, where $c > 0$ is a constant, then by substituting the result from Theorem~\ref{th:consensus}, we can derive the following expression
\begin{equation}
\text{\footnotesize
$\begin{aligned}
&\quad\frac{1}{K}\sum_{k=0}^{K-1}\mathbb{E}[||\nabla f(\overline{x}(k))||^2]\\
&\leq C_1\sqrt{\frac{1}{K}}+C_2\frac{1}{K}+ C_3\chi^2(\theta^2+\tau^2)+C_4d\sigma^2+C_5,
\end{aligned}$}
\nonumber
\end{equation}
where $C_1=\frac{2[f(\overline{x}(0))-f(x^*)]}{c}+\frac{2c\theta^2L}{\overline{n}}$, $C_2=\frac{3L^2(1+24\chi^2\overline{n}^2)}{\overline{n}^2}\cdot\frac{\lambda\Delta(0)+2c(8\theta^2+6\tau^2+2d\sigma^2)}{\lambda(1-\Lambda)}$, $C_3=\frac{96}{c\lambda(1-\Lambda)}+\frac{48}{\overline{n}}$, $C_4=\frac{8\chi^2}{c\lambda(1-\Lambda)}+\frac{3+12\chi^2}{\overline{n}}$, and $C_5=\frac{2\chi^2\Delta(0)}{c^2(1-\Lambda)}$.
\end{theorem}
\begin{proof}
According to Assumption \ref{ass:Lipschitz-continuous gradients}, we can obtain
\begin{equation}
\text{\footnotesize
\label{eq:th2_first}
$\begin{aligned}
\mathbb{E}[f(\overline{x}(k+1))]&\leq \mathbb{E}[f(\overline{x}(k))]+\frac{L}{2}\mathbb{E}[||\overline{x}(k+1)-\overline{x}(k)||^2]\\
&\quad+\mathbb{E}[\langle\nabla f(\overline{x}(k)),\overline{x}(k+1)-\overline{x}(k)\rangle].
\end{aligned}$}
\end{equation}
We will derive the convergence rate from the inequality \eqref{eq:th2_first}. First, we focus on the third term on the right-hand side, which can be decomposed as follows.
\begin{equation}
\text{\footnotesize
$\begin{aligned}
&\quad\mathbb{E}[\langle\nabla f(\overline{x}(k)),\overline{x}(k+1)-\overline{x}(k)\rangle]\\
&=\mathbb{E}[\gamma\langle\nabla f(\overline{x}(k)),\nabla f(\overline{x}(k);\xi(k))-\nabla f(\overline{x}(k))\\
&\quad+\frac{\overline{x}(k+1)-\overline{x}(k)}{\gamma}\rangle]\\
&\leq\gamma \mathbb{E}\bigg[\frac1{2}||\nabla f(\overline{x}(k);\xi(k))+\frac{\overline{x}(k+1)-\overline{x}(k)}{\gamma}||^2
-\frac{1}{2}||\nabla f(\overline{x}(k))||^2\\
&\quad -\frac{1}{2}||\nabla f(\overline{x}(k);\xi(k))-\nabla f(\overline{x}(k))+\frac{\overline{x}(k+1)-\overline{x}(k)}{\gamma}||^2\bigg].
\end{aligned}$}
\nonumber
\end{equation}
Subsequently, by combining it with the second term of the inequality \eqref{eq:th2_first}, we obtain
\begin{equation}
\text{\footnotesize
\label{eq:inner_product}
$\begin{aligned}
&\mathbb{E}[\langle\nabla f(\overline{x}(k)),\overline{x}(k+1)-\overline{x}(k)\rangle
+\frac{L}{2}||\overline{x}(k+1)-\overline{x}(k)||^2]\\
\leq &\mathbb{E}[\langle\nabla f(\overline{x}(k)),\overline{x}(k+1)-\overline{x}(k)\rangle\\
&+\gamma^2L||\nabla f(\overline{x}(k);\xi(k))-\nabla f(\overline{x}(k))||^2\\
&+\gamma^2L||\nabla f(\overline{x}(k);\xi(k))-\nabla f(\overline{x}(k))+\frac{\overline{x}(k+1)-\overline{x}(k)}{\gamma}||^2\\
=&\frac{\gamma}{2}\mathbb{E}[||\nabla f(\overline{x}(k);\xi(k))+\frac{\overline{x}(k+1)-\overline{x}(k)}{\gamma}||^2]+\frac{2\gamma^2L-\gamma}{2}\\
&\times||\nabla f(\overline{x}(k);\xi(k))-\nabla f(\overline{x}(k))+\frac{\overline{x}(k+1)-\overline{x}(k)}{\gamma}||^2\\
&+\frac{\gamma^2\theta^2L}{\overline{n}}-\frac{\gamma}{2}||\nabla f(\overline{x}(k))||^2\\
\overset{(\gamma\leq \frac{1}{2L})}{\leq} &\frac{\gamma}{2}||\nabla f(\overline{x}(k);\xi(k))+\frac{\overline{x}(k+1)-\overline{x}(k)}{\gamma}||^2+\frac{\gamma^2\theta^2L}{\overline{n}}\\
&-\frac{\gamma}{2}||\nabla f(\overline{x}(k))||^2.
\end{aligned}$}
\end{equation}
As for the inequality \eqref{eq:inner_product}, we first consider the first term on the right-hand side
\begin{equation}
\text{\footnotesize
$\begin{aligned}
&\quad \nabla f(\overline{x}(k);\xi(k))+\frac{\overline{x}(k+1)-\overline{x}(k)}{\gamma}\\
&=\nabla f(\overline{x}(k);\xi(k))-\frac{1}{\overline{n}}\sum\limits_{i\in \overline{\mathcal{V}}}[\nabla f_i(x_i(k);\xi_i(k))-\nabla f_i(x_i(k);\xi_i(k))]\\
&\quad+\frac{\overline{x}(k+1)-\overline{x}(k)}{\gamma}+\frac{1}{\overline{n}}\sum\limits_{i\in \overline{\mathcal{V}}}\eta_i(k)-
\frac{1}{\overline{n}}\sum\limits_{i\in \overline{\mathcal{V}}}\eta_i(k).
\end{aligned}$}
\nonumber
\end{equation}
Let $T_1=\mathbb{E}[||\nabla f(\overline{x}(k);\xi(k))-\frac{1}{\overline{n}}\sum\limits_{i\in \overline{\mathcal{V}}}\nabla f_i(x_i(k);\xi_i(k))||^2]$, and then we can obtain
\begin{equation}
\text{\footnotesize
$\begin{aligned}
T_1&=\mathbb{E}[||\frac{1}{\overline{n}}\sum_{i\in\overline{ \mathcal{V}}}\nabla f_i(\overline{x}(k);\xi(k))-\nabla f_i(x_i(k);\xi_i(k))||^2]\\
&\leq \frac{1}{\overline{n}} \sum_{i\in \overline{\mathcal{V}}} \mathbb{E}[||\nabla f_i(\overline{x}(k);\xi(k))-\nabla f_i(x_i(k);\xi_i(k))||^2]\\
&\leq \frac{L^2}{\overline{n}} \sum_{i\in \overline{\mathcal{V}}}\mathbb{E}[||x_i(k)-\overline{x}(k)||^2].
\end{aligned}$}
\nonumber
\end{equation}
Let $T_2=\mathbb{E}[||\frac{\overline{x}(k+1)-\overline{x}(k)}
{\gamma}+\frac{1}{\overline{n}}\sum\limits_{i\in \overline{\mathcal{V}}}\nabla f_i(x_i(k);\xi_i(k))+\frac{1}{\overline{n}}\sum\limits_{i\in \mathcal{V}}\eta_i(k)||^2]$. 
Recall that 
\begin{equation}
\text{\footnotesize
$\begin{aligned}
\overline{\tilde{x}}(k)=\overline{x}(k)-\frac{\gamma}{\overline{n}}\sum\limits_{i \in \overline{\mathcal{V}}}
\nabla f_i(x_i(k);\xi_i(k))-\frac{\gamma}{\overline{n}}\sum\limits_{i\in \overline{\mathcal{V}}}\eta_i(k).
\end{aligned}$}
\nonumber
\end{equation}
Based on this, we can derive 
\begin{equation}
\text{\footnotesize
$\begin{aligned}
T_2&=\mathbb{E}[||\frac{\overline{x}(k+1)-\overline{\tilde{x}}(k)}{\gamma}||^2]\\
&=\frac{1}{\gamma^2}\mathbb{E}[||\mathbf{1}^\top(\mathbf{M}(k)-\frac{1}{\overline{n}}\mathbf{1}\mathbf{1}^\top)\overline{\tilde{\mathbf{x}}}(k)||^2]\\
&\leq \frac{1}{\gamma^2\overline{n}^2}||\mathbf{M}^\top(k)\mathbf{1}-\mathbf{1}||^2\mathbb{E}[||(I_d-\frac{1}{\overline{n}}\mathbf{1}
\mathbf{1}^\top)\overline{\tilde{\mathbf{x}}}(k)||^2_F]\\
&=\frac{\chi^2}{\gamma^2\overline{n}}\sum_{i\in \overline{\mathcal{V}}}\mathbb{E}
||\tilde{x}_i(k)-\overline{\tilde{x}}(k)||^2,
\end{aligned}$}
\nonumber
\end{equation}
where $\chi^2=\frac{1}{\overline{n}}||\mathbf{M}^\top(k)\mathbf{1}-\mathbf{1}||^2$, describing how non doubly-stochastic $\mathbf{M}(k)$ is. Finally, let $T_3=\frac{1}{\overline{n}}\sum\limits_{i\in \overline{\mathcal{V}}}\eta_i(k)$, and we can obtain
\begin{equation}
\text{\footnotesize
$\begin{aligned}
T_3&=\mathbb{E}[||\frac{1}{\overline{n}}\sum_{i \in \overline{\mathcal{V}}}\eta_i(k)||^2]=\frac{1}{\overline{n}^2}\mathbb{E}[||\sum_{i \in \overline{\mathcal{V}}}\eta_i(k)||^2]\\
&=\frac{1}{\overline{n}^2}\sum_{i\in  \overline{\mathcal{V}}}d\sigma^2=\frac{d\sigma^2}{\overline{n}}
\end{aligned}$}
\nonumber
\end{equation}
Combining $T_1$, $T_2$, and $T_3$, we have
\begin{equation}
\text{\footnotesize
$\begin{aligned}
&\mathbb{E}[||\nabla f(\overline{x}(k))+\frac{\overline{x}(k+1)-\overline{x}(k)}{\gamma}||^2]\\
\leq&3T_1+3T_2+3T_3\\
\leq&\frac{3L^2}{\overline{n}} \sum_{i\in \overline{\mathcal{V}}}\mathbb{E}[||x_i(k)-\overline{x}(k)||^2]+\frac{3d\sigma^2}{\overline{n}}\\
& +\frac{3\chi^2}{\gamma^2\overline{n}}\sum_{i\in \overline{\mathcal{V}}}\mathbb{E}[||\tilde{x}_i(k)-\overline{\tilde{x}}(k)||^2].
\end{aligned}$}
\nonumber
\end{equation}
Substituting inequality \eqref{eq:gr_compress} into the expression and setting $v=\frac{1}{2}$, we derive
\begin{equation}
\text{\footnotesize
\label{eq:two_norm}
$\begin{aligned}
&\mathbb{E}[||\nabla f(\overline{x}(k))+\frac{\overline{x}(k+1)-\overline{x}(k)}{\gamma}||^2]\\
\leq& 3(L^2+\frac{2\chi^2}{\gamma^2}+24\chi^2L^2)]\mathbb{E}[\Delta(k)]\\
&+\frac{3}{\overline{n}}(d+4\chi^2d)\sigma^2+\frac{12\chi^2}{\overline{n}}(4\theta^2+3\tau^2)\\
\end{aligned}$}
\end{equation}
Combining inequalities \eqref{eq:th2_first}, \eqref{eq:inner_product}, and \eqref{eq:two_norm}, we have
\begin{equation}
\text{\footnotesize
$\begin{aligned}
&\mathbb{E}[f(\overline{x}(k+1))]\\
\leq& \mathbb{E}[f(\overline{x}(k))]+\frac{\gamma^2\theta^2L}{\overline{n}}+\frac{\gamma}{2}\bigg[3(L^2+\frac{2\chi^2}{\gamma^2}+24\chi^2L^2)\mathbb{E}[\Delta(k)]\\
&+\frac{3}{\overline{n}}(d+4\chi^2d)\sigma^2- \mathbb{E}[||\nabla f(\overline{x}(k))||^2]+\frac{12\chi^2}{\overline{n}}(4\theta^2+3\tau^2)\bigg].
\end{aligned}$}
\nonumber
\end{equation}
By transposing the terms on both sides of the inequality, we obtain the following recurrence relation
\begin{equation}
\label{eq:induction_grad}
\text{\footnotesize
$\begin{aligned}
&\quad \mathbb{E}[||\nabla f(\overline{x}(k))||^2]\\&\leq \frac{2\mathbb{E}[f(\overline{x}(k))-f(\overline{x}(k+1))]}{\gamma}+\frac{12\chi^2}{\overline{n}}(4\theta^2+3\tau^2)+\frac{2\gamma\theta^2L}{\overline{n}}\\
&\quad+3(L^2+\frac{2\chi^2}{\gamma^2}+24\chi^2L^2)\mathbb{E}[\Delta(k)]+\frac{3}{\overline{n}}(d+4\chi^2d)\sigma^2.
\end{aligned}$}
\end{equation}
Taking the sum over $k=0,1,\ldots,K-1$, we obatin
\begin{equation}
\label{eq:th2_result}
\text{\footnotesize
$\begin{aligned}
&\frac{1}{K}\sum_{k=0}^{K-1}\mathbb{E}[||\nabla f(\overline{x}(k))||^2]\\
\leq &\frac{2\mathbb{E}[f(\overline{x}(0))-f(\overline{x}(K))]}{\gamma K}+ \frac{12\chi^2}{\overline{n}}(4\theta^2+3\tau^2)+\frac{2\gamma\theta^2L}{\overline{n}}\\
&+\frac{3(\frac{L^2}{\overline{n}^2}+\frac{2\chi^2}{\gamma^2}+24\chi^2L^2)}{K}\sum_{k=0}^{K-1} \mathbb{E}[\Delta(k)]+\frac{3}{\overline{n}}(d+4\chi^2d)\sigma^2\\
\leq &\frac{2\mathbb{E}[f(\overline{x}(0))-f(x^*)]}{\gamma K}+ \frac{12\chi^2}{\overline{n}}(4\theta^2+3\tau^2)+\frac{2\gamma\theta^2L}{\overline{n}}\\
&+\frac{3(\frac{L^2}{\overline{n}^2}+\frac{2\chi^2}{\gamma^2}+24\chi^2L^2)}{K}\sum_{k=0}^{K-1} \mathbb{E}[\Delta(k)]+\frac{3}{\overline{n}}(d+4\chi^2d)\sigma^2.
\end{aligned}$}
\end{equation}
Substituting $\gamma = c\cdot \sqrt{\frac{1}{K}}$ and the result of Theorem \ref{th:consensus} into \eqref{eq:th2_result}, we have
\begin{equation}
\text{\footnotesize
$\begin{aligned}
&\quad\frac{1}{K}\sum_{k=0}^{K-1}\mathbb{E}[||\nabla f(\overline{x}(k))||^2]\\
&\leq \frac{2[f(\overline{x}(0))-f(x^*)]}{c}\sqrt{\frac{1}{K}}+3\bigg[\frac{L^2}{\overline{n}^2K}+\frac{2\chi^2}{c^2}+\frac{24\chi^2L^2}{K}\bigg]\\
&\quad\times \bigg[\frac{\Delta(0)}{1-\Lambda}+\frac{2c(8\theta^2+6\tau^2+2d\sigma^2)}{\lambda(1-\Lambda)}\bigg]+\frac{2c\theta^2L}{\overline{n}}\sqrt{\frac{1}{K}}\\
&\quad +\frac{12\chi^2(4\theta^2+3\tau^2)+3d(1+4\chi^2)\sigma^2}{\overline{n}}.
\end{aligned}$}
\nonumber
\end{equation}
By further simplification, we obatin 
\begin{equation}
\text{\footnotesize
$\begin{aligned}
&\quad\frac{1}{K}\sum_{k=0}^{K-1}\mathbb{E}[||\nabla f(\overline{x}(k))||^2]\\
&\leq C_1\sqrt{\frac{1}{K}}+C_2\frac{1}{K}+ C_3\chi^2(\theta^2+\tau^2)+C_4d\sigma^2+C_5,
\end{aligned}$}
\nonumber
\end{equation}
where $C_1=\frac{2[f(\overline{x}(0))-f(x^*)]}{c}+\frac{2c\theta^2L}{\overline{n}}$, $C_2=\frac{3L^2(1+24\chi^2\overline{n}^2)}{\overline{n}^2}\cdot\frac{\lambda\Delta(0)+2c(8\theta^2+6\tau^2+2d\sigma^2)}{\lambda(1-\Lambda)}$, $C_3=\frac{96}{c\lambda(1-\Lambda)}+\frac{48}{\overline{n}}$, $C_4=\frac{8\chi^2}{c\lambda(1-\Lambda)}+\frac{3+12\chi^2}{\overline{n}}$, and $C_5=\frac{2\chi^2\Delta(0)}{c^2(1-\Lambda)}$.
\end{proof}
\begin{remark}
In our analysis, we do not derive an upper bound for $\mathbb{E}[|f(\overline{x}(K))-f^*|]$ to characterize convergence, as we consider the non-convex loss function. Under the assumptions of this paper, convergence to a global minimum cannot be guaranteed under existing assumptions. However, we can analyze the norm of the gradient to characterize the convergence rate toward first-order stationary points. This approach is widely adopted in decentralized DML under the assumption of non-convex loss functions, such as in \cite{lian2017can, fang2024byzantine, wu2023byzantine, yang2024byzantine}.
\end{remark}

\par  The final result comprises several terms that fall into two categories: decaying terms and constant terms. For the decaying term, as the number of iterations increases, $C_1\sqrt{\frac{1}{K}}$ gradually becomes the dominant term. This behavior is consistent with the results of decentralized DML without privacy preservation and Byzantine resilience, as discussed in \cite{koloskova2020unified}. The constant term, on the other hand, indicates that exact convergence to first-order stationary points is unachievable; instead, the algorithm converges to a neighborhood of such points due to various sources of error. We proceed with a detailed analysis, noting that three main types of terms are primarily responsible for the errors.
\begin{enumerate}
\item $\theta^2$ and $\tau^2$: These two factors correspond to the variance of the gradient: $\theta^2$ represents the error introduced by random sampling of the dataset, 
while $\tau^2$ accounts for the differences in loss functions and local datasets across different normal agents. The larger these two coefficients, the greater the error.
\item $\sigma^2$: This factor corresponds to the strength of the added noise. The larger the $\sigma$, the greater the error.
\item $\chi^2$: This factor describes how the matrix $\mathbf{M}(k)$ non-doubly stochastic is. If $\mathbf{M}(k)$ is a doubly-stochastic matrix, $\chi=0$, eliminating the influence of $\theta^2$ and $\tau^2$. The larger the $\chi$, the more it amplifies the influence of the first two factors, resulting in greater error.
\end{enumerate}
\par The improvement in accuracy is due to the elimination of an error term. Previous works \cite{wu2023byzantine,yang2024byzantine,he2022byzantine} fail to obtain a point within the normal agents' convex hull, introducing inherent approximation errors. Notably, while these works claim to construct a doubly stochastic matrix without $\chi^2$ term, such construction requires each normal agent to know the complete knowledge of all agent' degrees in the network. This becomes particularly challenging in \emph{time-varying directed networks} with Byzantine agents that may maliciously manipulate information, while also incurring high computational complexity and yielding wrong matrices. Consequently, the constructed doubly stochastic matrices often perform poorly in practice. In contrast, Algorithm \ref{alg:PP-SGD-ADRC} achieves significantly better final learning accuracy, as will be demonstrated in subsequent simulations.

\subsection{Privacy Guarantee}
In this part, we provide the privacy guarantee of the proposed ImprovDML framework. We use CGP to assess the privacy-preservation capabilities. Our objective is to preserve the dataset $\mathcal{D}_i$ of each normal agent $i$. Before providing the details, we first introduce $(\alpha,\epsilon)$-R\'enyi DP ($(\alpha,\epsilon)$-RDP) \cite{mironov2017renyi}.
\begin{definition}
($(\alpha,\epsilon)$-RDP). Given spaces $U$, $V$, and $\epsilon\in \mathbb{R}\geq 0$, a randomized function $\mathcal{M}:U\rightarrow V$ is said to satisfy $(\alpha,\epsilon)$-RDP, 
if for any pair of inputs $x,x'\in U$ and all $\alpha>1$, it holds that   
\begin{equation}
D_\alpha (\mathcal{M}(x)||\mathcal{M}(x')\leq \epsilon.
\nonumber
\end{equation}
\end{definition}
Then, we define the global sensitivity of a function $f$.
\begin{definition}
($l_p$-sensitivity). Given spaces $U$ and $V$, the $l_p$-sensitivity of a function $f:U\rightarrow V$ is defined by 
\begin{equation}
\mathcal{S}_p(f)=\max\limits_{x,x'}||f(x)-f(x')||_p,
\nonumber
\end{equation}
where $x$, $x' \in U$. 
\end{definition}
\par We give the following lemma from \cite{wang2019subsampled,mironov2017renyi,xu2021dp}.
\begin{lemma}
\cite{wang2019subsampled,mironov2017renyi,xu2021dp} Given a mechanism $\mathcal{M}_{\text{Gau}}$ that adds Gaussian noise to a function $f(x)$ with covariance matrix as $\sigma^2I_d$, it satisfies $(\alpha, \alpha \frac{\mathcal{S}_2^2(f)}{2\sigma^2})$-RDP. In addition, if $\mathcal{M}_{\text{Gau}}$ is applied to a subset of samples using uniform sampling without replacement $Q_\zeta$, then $\mathcal{M}_{\text{Gau}}^{Q_\zeta}$ that applies $\mathcal{M}_{\text{Gau}}\circ Q_\zeta$ obeys $(\alpha,\alpha\frac{ 5\zeta^2 \mathcal{S}_2^2(f)}{\sigma^2})$-RDP when $\frac{\sigma^2}{\mathcal{S}_2^2(f)}\geq 1.5$ and $\alpha\leq\log(\frac{S_{2}^{2}(f)}{\zeta(S_{2}^{2}(f)+\sigma^{2})})$ with $0<\zeta<1$ denoting the subsampling rate.
\end{lemma}
\par Actually, we observe that subsampling only introduces a coefficient to the original privacy parameter $\epsilon = \alpha \frac{\mathcal{S}_2^2(f)}{2\sigma^2}$, thereby enhancing 
the privacy-preserving capability of the mechanism. During the local parameter training process in \eqref{eq:add_noise}, each agent employs SGD to perform random subsampling on its local dataset. This procedure strengthens the privacy preservation of the original algorithm, resulting in a smaller $\epsilon$. This technique is known as privacy amplification \cite{bassily2014private}. We then establish the upper bound of $\rho$-CGP for Algorithm \ref{alg:PP-SGD-ADRC}.
\begin{theorem}
\label{th:cgp_learning}
(Privacy Guarantee). With Assumption \ref{ass:lipstichz-input}, the dynamics of the normal agent $i\in \overline{\mathcal{V}}$ under Algorithm \ref{alg:PP-SGD-ADRC} satisfy $ \rho_i $-CGP, where $ \rho_i = \frac{5K\zeta^2_i L'^2}{\sigma^2} $. This holds when $ \frac{\sigma^2}{L'^2} \geq 1.5 $ and $ \alpha \leq \log\left(\frac{L'^2}{\zeta_i(L'^2+\sigma^2)}\right) $, where $ \zeta_i $ denotes the subsampling rate.
\end{theorem}
\begin{proof}
We first consider the result of $(\alpha,\epsilon)$-RDP. For a normal agent $i\in\overline{\mathcal{V}}$, we denote its local dataset as $\mathcal{D}_i$, and we can obtain 
\begin{equation}
D_\alpha(\mathcal{M}_{\text{Gau}}^{Q_\zeta}(\mathcal{D}_i)||\mathcal{M}_{\text{Gau}}^{Q_\zeta}(\mathcal{D}_i'))\leq \alpha\frac{ 5\zeta^2_i \mathcal{S}_2^2(\nabla f_i)}{\sigma^2},
\nonumber
\end{equation}
where $\mathcal{D}_i$ and $\mathcal{D}_i'$ only differ in one entry. Expanding $D_\alpha(\mathcal{M}_{\text{Gau}}^{Q_\zeta}(\mathcal{D}_i)||\mathcal{M}_{\text{Gau}}^{Q_\zeta}(\mathcal{D}_i'))$, 
we obatin 
\begin{equation}
\begin{aligned}
&D_\alpha(\mathcal{M}_{\text{Gau}}^{Q_\zeta}(\mathcal{D}_i)||\mathcal{M}_{\text{Gau}}^{Q_\zeta}(\mathcal{D}_i'))\\
=&A D_\alpha\bigg(\text{Gau}\bigg(\nabla f_i(x_i(k);\xi_i(k)),\sigma^2I_d\bigg)\\
&\quad||\text{Gau}\bigg(\nabla f_i(x_i(k);\xi_i'(k)),\sigma^2I_d\bigg)\bigg)\\
=&\frac{A\alpha ||\nabla f_i(x_i(k);\xi_i(k))-\nabla f_i(x_i(k);\xi_i'(k))||^2}{2\sigma^2}\\
\leq &\frac{A\alpha \mathcal{S}_2^2(\nabla f_i )}{2\sigma^2},
\end{aligned}
\nonumber
\end{equation}
where $A$ is a constant related to the subsampling and the Gaussian mechanism. Then, we shift our focus to CGP. Unlike traditional definitions, CGP does not require $\mathcal{D}_i$ and 
$\mathcal{D}_i'$ to differ by only a single entry; instead, it considers any arbitrary pair of $\mathcal{D}_i$ and $\mathcal{D}_i'$. Based on this broader perspective, we 
derive the following formula
\begin{equation}
\begin{aligned}
&D_\alpha(\mathcal{M}_{\text{Gau}}^{Q_\zeta}(\mathcal{D}_i)||\mathcal{M}_{\text{Gau}}^{Q_\zeta}(\mathcal{D}_i'))\\
=&\frac{A\alpha ||\nabla f_i(x_i(k);\xi_i(k))-\nabla f_i(x_i(k);\xi_i'(k))||^2}{2\sigma^2}\\
\leq&\frac{A\alpha L'^2}{2\sigma^2} ||\xi_i(k)-\xi_i'(k)||^2\\
\leq &\frac{A\alpha L'^2}{2\sigma^2} ||\mathcal{D}_i-\mathcal{D}_i'||^2
\end{aligned}
\nonumber
\end{equation}
We observe that the only differences in the parameters are $L'$ and $\mathcal{S}_2^2(\nabla f_i)$. Therefore, we can obtain the upper bound of $\rho$-CGP for each iteration, which is $\frac{5\zeta_i^2L'^2}{\sigma^2}$. Then, by using the advanced composition of CGP \cite{liang2023concentrated}, we derive $\rho_i=\frac{5K\zeta_i^2L'^2}{\sigma^2}$ along with other constraints on the parameters after $K$ iterations.
\end{proof}
\begin{remark}
Assumption $7$ is a necessary condition for achieving $\rho$-CGP \cite{liang2023concentrated}. Similar to Assumption \ref{ass:Lipschitz-continuous gradients}, considering that both model parameters and inputs are typically bounded, and that many modern machine learning methods incorporate the manifold hypothesis, this assumption is reasonable. Moreover, in a class of privacy attacks known as gradient inversion attacks, stronger assumptions are often made. Specifically, it is assumed that the gradient of the loss function $\nabla f_i$ with respect to the input $\xi_i$ is differentiable, and that the gradients are bounded to ensure the success of the attack \cite{zhu2019deep, zhao2020idlg, hatamizadeh2022gradvit}.
\end{remark}

\subsection{Imporved Trade-off between Privacy and Accuracy}\label{sec:tradeoff}
\par In this subsection, we first analyze the trade-off between privacy and accuracy under $\rho$-CGP and then show that this trade-off is more accurate compared to that under $(\varepsilon,\delta)$-DP (local DP).
\subsubsection{Trade-off under $\rho$-CGP}
\par We first illustrate the trade-off under $\rho$-CGP. It is evident from Theorems \ref{th:convergence} and \ref{th:cgp_learning} that a larger noise amplitude parameter $\sigma$ results in lower learning accuracy and a smaller $\rho_i$, indicating stronger privacy preservation. Therefore, the key to balancing privacy and accuracy lies in the choice of the parameter $\sigma$. While the above provides a qualitative understanding, we are more interested in the quantitative relationship between these two aspects. As stated in Theorem \ref{th:convergence}, the learning error term is proportional to $\sigma^2$. Although the analytical form of the trade-off rate between privacy and accuracy is not explicitly obtainable, we can assert that this rate is proportional to $\frac{\mathrm{d}\rho_i}{\mathrm{d}(-\sigma^2)}$, i.e.,
\begin{equation}
\frac{\mathrm{d}\rho_i}{\mathrm{d}(-\sigma^2)}=\frac{5K\zeta_i^2L'^2}{\sigma^4},
\nonumber
\end{equation}
A smaller value of this derivative indicates that the privacy parameter decreases more slowly as the accuracy improves, implying a more accurate trade-off.
\subsubsection{Comparison with $(\varepsilon,\delta)$-DP}
\par We then compare it with the $(\varepsilon,\delta)$-DP. We first present the conclusion of the comparison: CGP offers a more accurate trade-off from the following two perspectives.
\begin{enumerate}
    \item CGP allows for more flexible adjustment of the privacy preservation level, offering a broader range of noise parameter options.
    \item CGP enables a more accurate trade-off, where improvements in model accuracy do not lead to a significant degradation in the privacy parameter.
\end{enumerate}
The result under $(\varepsilon,\delta)$-DP is given as follows.
\begin{lemma}
($(\varepsilon,\delta)$-DP Guarantee). If $||\nabla f_i||$ is upper bounded by $G>0$, the dynamics of the normal agent $i \in \overline{\mathcal{V}}$ under Algorithm \ref{alg:PP-SGD-ADRC} satisfy $(\varepsilon_i,\delta)$-DP with $\varepsilon_i=\frac{20G^2\zeta_i^2K}{\sigma^2}+\frac{2G\zeta_i}{\sigma}\sqrt{20K\log(\frac{1}{\delta})}$, where $0<\delta<1$.
\end{lemma}
\par Considering that a direct comparison between $(\varepsilon,\delta)$-DP and $\rho$-CGP is not feasible, we first give the definition of $(\varepsilon_\mathrm{geo},\delta,r)$-GP \cite{liang2023concentrated}.
\begin{definition}
($(\varepsilon_{\mathrm{geo}},\delta,r)$-GP). Given spaces $U,V$, and $\varepsilon_{\mathrm{geo}},\delta \in \mathbb{R}\geq 0$. a randomized function $\mathcal{M}:U\rightarrow V$ is $(\varepsilon_{\mathrm{geo}},\delta,r)$-GP, if for any inputs $x,x'\in U$ satisfying $||x-x'||\leq r$, we have 
\begin{equation}
\mathrm{P}\{\mathcal{M}(x)\in S\}\leq e^{\varepsilon_{\mathrm{geo}}\cdot||x-x'||}\mathrm{P}\{\mathcal{M}(x')\in S\}+\delta.
\nonumber
\end{equation}
\end{definition}

\par According to Lemma $3.7$ in \cite{liang2023concentrated}, $\rho$-CGP and $(\varepsilon_\mathrm{geo},\delta,r)$-GP are mutually convertible. Given a value of $\rho$, the corresponding $(\varepsilon_\mathrm{geo},\delta,r)$ parameters satisfy the bound $\varepsilon_\mathrm{geo} \geq \rho r + 2\sqrt{\rho \log\left(\frac{1}{\delta}\right)}$. Therefore, Algorithm \ref{alg:PP-SGD-ADRC} satisfies $(\varepsilon_{\mathrm{geo},i},\delta,r)$-GP, where $\varepsilon_{\mathrm{geo},i}=\frac{5K\zeta^2_iL'^2}{\sigma^2}r+\frac{\zeta_i L'}{\sigma}\sqrt{20K\log(\frac{1}{\delta})}$. 
\par $(\varepsilon,\delta)$-DP is a highly stringent privacy-preserving mechanism that is suitable for scenarios requiring strong privacy guarantees. However, in the context of machine learning, it suffers from the following limitations: DP enforces that the gradients corresponding to any pair of inputs must be similar, making it impossible for an external observer to distinguish between them based on gradient information. However, this strict requirement often significantly degrades gradient utility, leading to reduced model accuracy. In practical applications, achieving a small privacy parameter $\varepsilon$ under DP requires adding a large noise scale $\sigma$, which often results in significantly reduced learning accuracy. Nevertheless, in practice,  the gradients corresponding to most input pairs differ far less than in the worst-case analysis.This observation implies that even a relatively small noise scale $\sigma$ can provide strong privacy preservation while preserving high model accuracy, as will be demonstrated in the simulation section. Unfortunately, the standard $(\varepsilon, \delta)$-DP is unable to effectively capture or describe this detailed form of privacy preservation.

\par GP introduces a distance-based dimension by requiring gradient similarity only within a ball of arbitrary radius centered around a given input. Therefore, GP addresses the limitations of $(\varepsilon,\delta)$-DP by offering a more fine-grained characterization of privacy preservation, based on the notion of distance between inputs. As a result, $\varepsilon_i$ is typically much larger than $\varepsilon_{\mathrm{geo},i}$, making it difficult to select an appropriate value for $\sigma$. We can also observe from the expressions of $\varepsilon_i$ and $\varepsilon_{\mathrm{geo},i}$ that their difference lies solely in the use of $2G$ versus $L'$, where $2G$ is typically much larger than $L'$. This disparity directly accounts for the difference between the two results. Therefore, GP can ensure that the privacy parameters remain relatively small under a reasonable noise scale, thereby allowing greater flexibility in the selection of noise parameter $\sigma$.
\par As for the trade-off rate, we can derive 
\begin{equation}
\text{\footnotesize
$\begin{aligned}
\frac{\mathrm{d}\varepsilon_i}{\mathrm{d}(-\sigma^2)}=\frac{20G^2\zeta_i^2K}{\sigma^4}+\frac{G\zeta_i}{\sigma^3}\sqrt{20K\log(\frac{1}{\delta})},
\end{aligned}$}
\nonumber
\end{equation}
and 
\begin{equation}
\text{\footnotesize
$\begin{aligned}
\frac{\mathrm{d}\varepsilon_{\mathrm{geo},i}}{\mathrm{d}(-\sigma^2)}=\frac{5K\zeta_i^2L'^2}{\sigma^4}r+\frac{\zeta_iL'}{2\sigma^3}\sqrt{20K\log(\frac{1}{\delta})}.
\end{aligned}$}
\nonumber
\end{equation}
\par It can be observed that, with respect to the trade-off rate, the proportional coefficients remain unchanged. Therefore, GP provides a more accurate trade-off between learning accuracy and privacy preservation. Specifically, variations in the noise scale $\sigma$ lead to smaller changes in $\varepsilon_\mathrm{geo}$, compared to the $\varepsilon$ in DP.
\begin{remark}
 We introduce the definition of GP due to its conceptual similarity to $(\varepsilon,\delta)$-DP, facilitating a direct comparison. However, our choice is CGP, as it offers several advantages better suited to DML settings. First, CGP avoids the use of a ``$\delta$'' parameter, which permits a small probability of full privacy leakage. Furthermore, the inclusion of the $\delta$ term would lead to a substantial increase in $\varepsilon_\mathrm{geo}$. More importantly, CGP supports advanced composition, which ensures that privacy loss grows more slowly with the number of iterations, making it particularly suitable for iterative learning algorithms. These advantages collectively suggest that the overall privacy cost $\rho$ under CGP is typically much smaller than the corresponding $\varepsilon_{\mathrm{geo}}$ in GP. Therefore, CGP further amplifies the advantages of GP discussed above, offering a broader range of noise parameter choices and a more accurate, and thus better, trade-off between privacy and accuracy.
\end{remark}

\section{NUMERICAL SIMULATIONS}\label{sec:sim}
In this section, we evaluate the performance of the proposed ImprovDML framework. We conduct simulations for multi-robot learning scenarios \cite{li2022byzantine} shown in Fig. \ref{fig:multirobot}. Specifically, we consider a binary classification task in a two-dimensional space based on XOR data, where each robot receives two-dimensional sensory inputs to determine the current state. Such a setting is representative of practical applications including fault detection in factories and intrusion detection in security systems.
\begin{comment}
\footnote{A concrete example could be fault detection in a factory, where each robot collects data such as vibration frequency and noise levels to determine whether a fault has occurred.}. 
\end{comment}

\par We first present the data distribution. The data set is synthesized by sampling the data points uniformly from the two-dimensional space $[-8,8] \times [-8,8]$ with a padding of $0.5$ shown in Fig. \ref{fig:data_distri}. Using the XOR operation as a classification criterion, we partition all points into two classes: positive and negative instances. We also consider the presence of data noise, where the labels of some points may be flipped.
\begin{comment}
We apply a directional padding operation to each data point to enhance the discriminative boundary between classes. Specifically, for each feature dimension, a constant padding value $\mathrm{pad}=0.5$ is added or subtracted based on the feature's sign:
\begin{equation}
\begin{aligned}
\mathrm{Feature}\text{ }1 &\leftarrow \mathrm{Feature}\text{ }1 + \mathrm{pad} \cdot \operatorname{sgn}(\mathrm{Feature}\text{ }1), \\
\mathrm{Feature}\text{ }2&\leftarrow \mathrm{Feature}\text{ }2 + \mathrm{pad} \cdot \operatorname{sgn}(\mathrm{Feature}\text{ }2),
\end{aligned}
\nonumber
\end{equation}
where $\operatorname{sgn}(\cdot)$ denotes the sign function that returns $+1$ for positive inputs and $-1$ for negative values. This strategic displacement effectively creates a margin around the coordinate axes, ensuring clearer separation between the four characteristic quadrants of the XOR decision boundary.
\end{comment}
\begin{figure}[H]
  \begin{subfigure}{0.48\linewidth}
    \centerline{\includegraphics[width=0.8\linewidth]{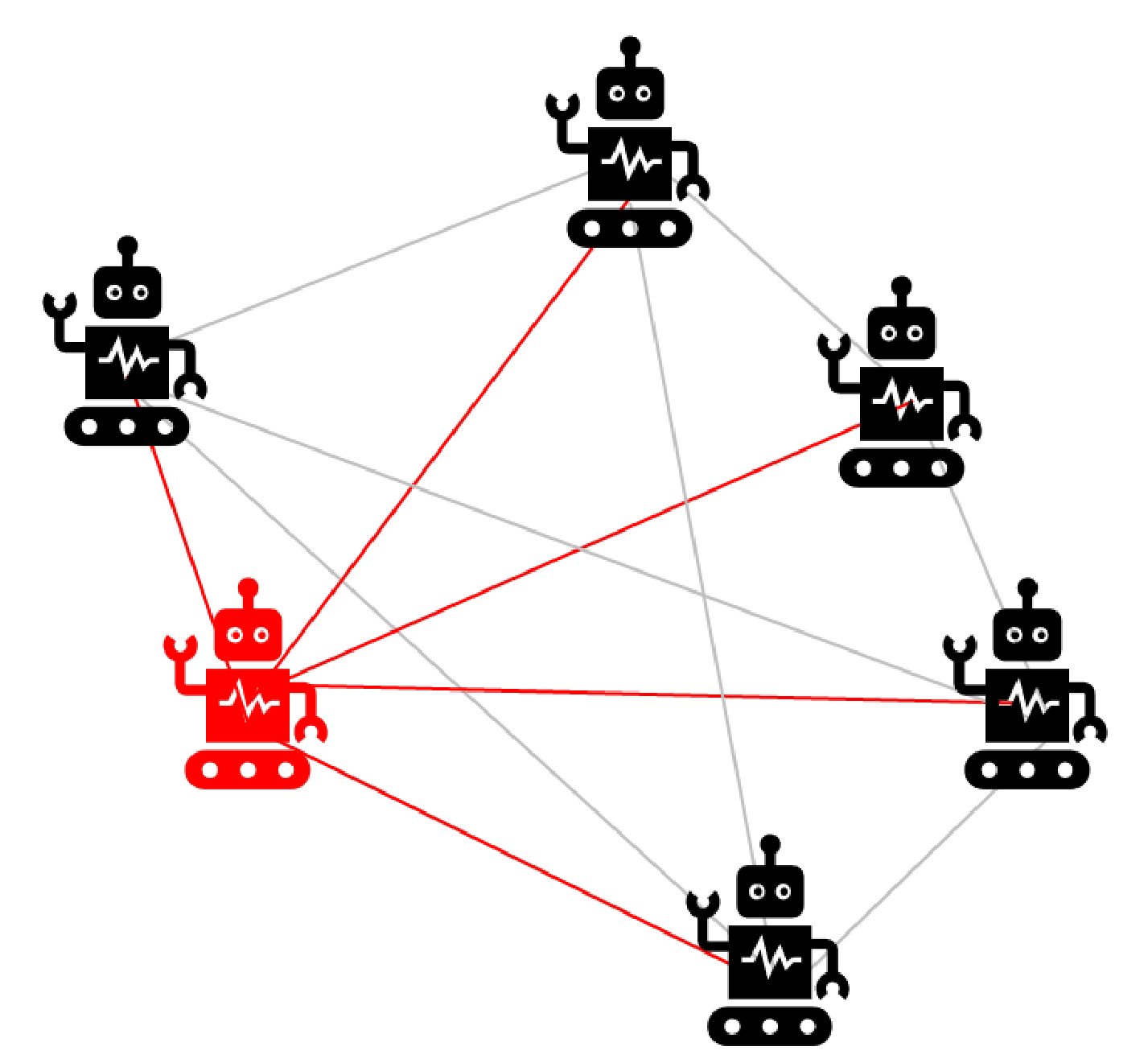}}
    \vspace{-5pt}
        \caption{Multi-robot learning}  
    \label{fig:multirobot}  
  \end{subfigure}
  \hfill
  \begin{subfigure}{0.48\linewidth}
    \centerline{\includegraphics[width=1.0\linewidth]{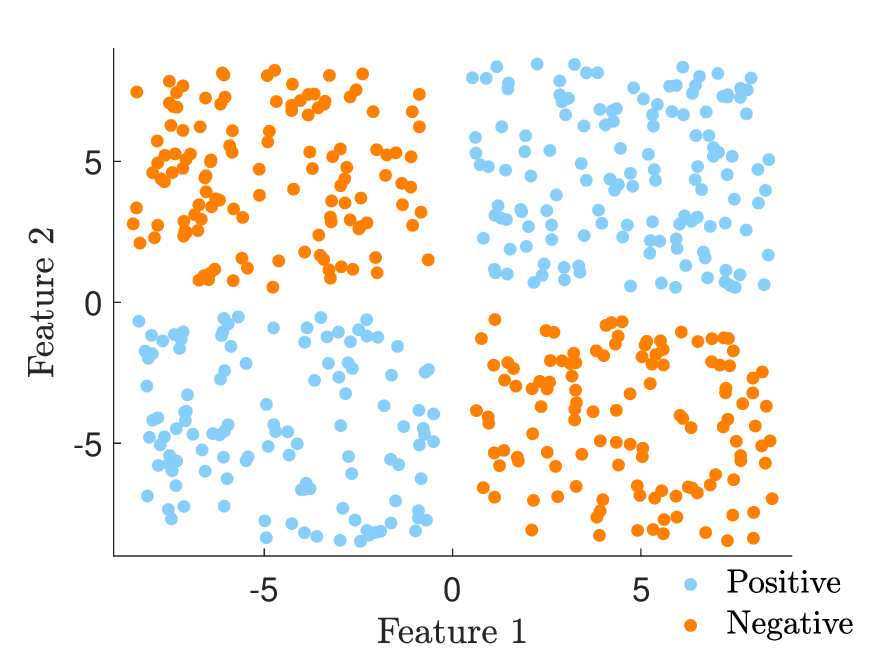}}
    \vspace{-5pt}
    \caption{Data distribution}
    \label{fig:data_distri}
  \end{subfigure}
  \vspace{-2pt}
  \caption{Simulation scenarios}
\end{figure}

% \begin{figure}[H]
%   \begin{subfigure}{0.48\linewidth}
%     \centerline{\includegraphics[width=1.0\linewidth]{figures_xor/real_data.eps}}
%     \vspace{-5pt}
%         \caption{True data distribution}  
%     \label{fig:real_data}  
%   \end{subfigure}
%   \hfill
%   \begin{subfigure}{0.48\linewidth}
%     \centerline{\includegraphics[width=1.0\linewidth]{figures_xor/noisy_data.eps}}
%     \vspace{-5pt}
%     \caption{Noisy data distribution}
%     \label{fig:noisy_data}  
%   \end{subfigure}
%   \vspace{-2pt}
%   \caption{Data distribution}
% \end{figure}
\par We consider a network of $14$ robots with a connected graph topology, where one of them acts as the Byzantine agent. The Byzantine agent transmits arbitrary and different messages to different neighbors. The sizes of the train dataset for each normal agent are 
$1122$,\allowbreak $1315$,\allowbreak $1521$,\allowbreak $1400$,\allowbreak $1369$,\allowbreak $1255$,\allowbreak 
$1239$,\allowbreak $1160$,\allowbreak $1138$,\allowbreak $1588$,\allowbreak $1550$,\allowbreak $1384$,\allowbreak and $1238$, respectively. As for the test dataset, the detailed sizes are 
$312$,\allowbreak $309$,\allowbreak $300$,\allowbreak $286$,\allowbreak $283$,\allowbreak $313$,\allowbreak 
$312$,\allowbreak $294$,\allowbreak $291$,\allowbreak $283$,\allowbreak $292$,\allowbreak $204$,\allowbreak and $299$, respectively.
\par The neural network architecture consists of an input layer with $2$ neurons, a single hidden layer comprising $4$ neurons with the tanh activation function, and an output layer with a single neuron using the sigmoid activation function. In the training process, the batch size is set to $16$. The local loss functions of the normal agents are identical, which are binary cross entropy loss functions. 
\par To compare our algorithm with other existing ones, we need to get the upper bound of the gradient's $2$-norm $||\nabla f_i||$ to calculate $(\varepsilon,\delta)$-DP, and the Lipschitz constant $L'$ in Assumption \ref{ass:lipstichz-input}. For the upper bound of $||\nabla f_i||$, we conduct local simulations for each normal agent. Before each SGD update, we compute the maximum possible gradient for each training data point under the current model parameters and record it. A total of $3000$ iterations are run with a batch size of $16$, and the maximum observed value is $9.2$, i.e., $G=9.2$. The trend of change across iterations is shown in Fig. \ref{fig:gradient_norm}. Therefore, we incorporate gradient clipping with a threshold of $9.2$ into the algorithm proposed by \cite{ye2024tradeoff}. As for the Lipschitz constant $L'$, it is similarly obtained through simulations. We reformulate the problem as finding the maximum value of the second-order gradient, i.e., the maximum of $\nabla_{\xi_i} \nabla f_i$. Before each model parameter update, we compute the derivative of $\nabla f_i$ with respect to all inputs in the training dataset (which can be directly implemented using PyTorch's autograd function). We then calculate the norm of the gradient and determine the maximum norm, which corresponds to the Lipschitz constant $L'$. The estimated Lipschitz constant across iterations is shown in Fig. \ref{fig:gradient_norm} and we have $L'=0.84$. It is worth noting that we set the number of iterations to $3000$ in both experiments, as this is sufficient for each normal agent to train a highly accurate and converged local model.
\begin{figure}[H]
  \begin{subfigure}{0.48\linewidth}
    \centerline{\includegraphics[width=1.0\linewidth]{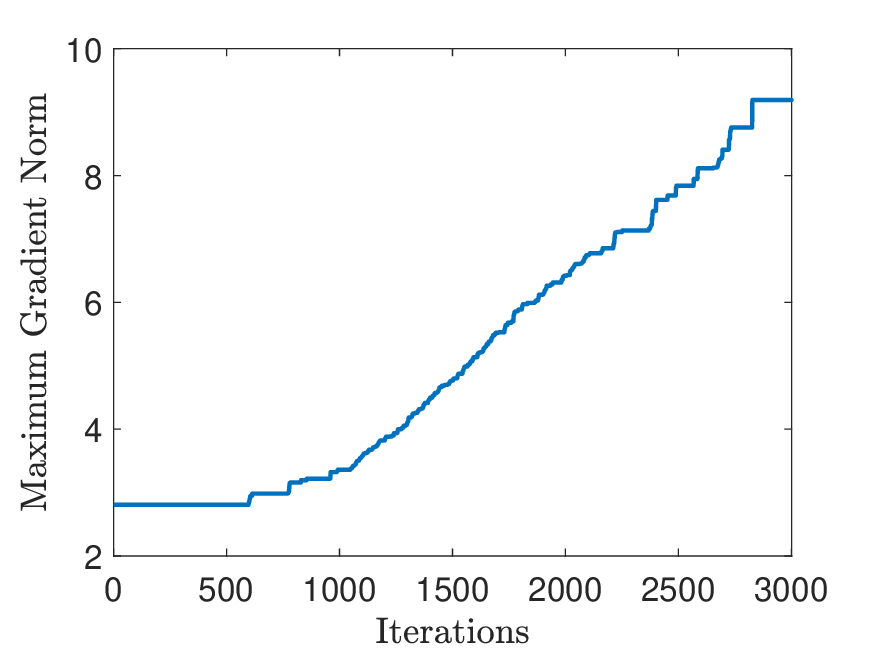}}
    \vspace{-5pt}
        \caption{Maximum gradient norm}  
    \label{fig:gradient_norm}  
  \end{subfigure}
  \hfill
  \begin{subfigure}{0.48\linewidth}
    \centerline{\includegraphics[width=1.0\linewidth]{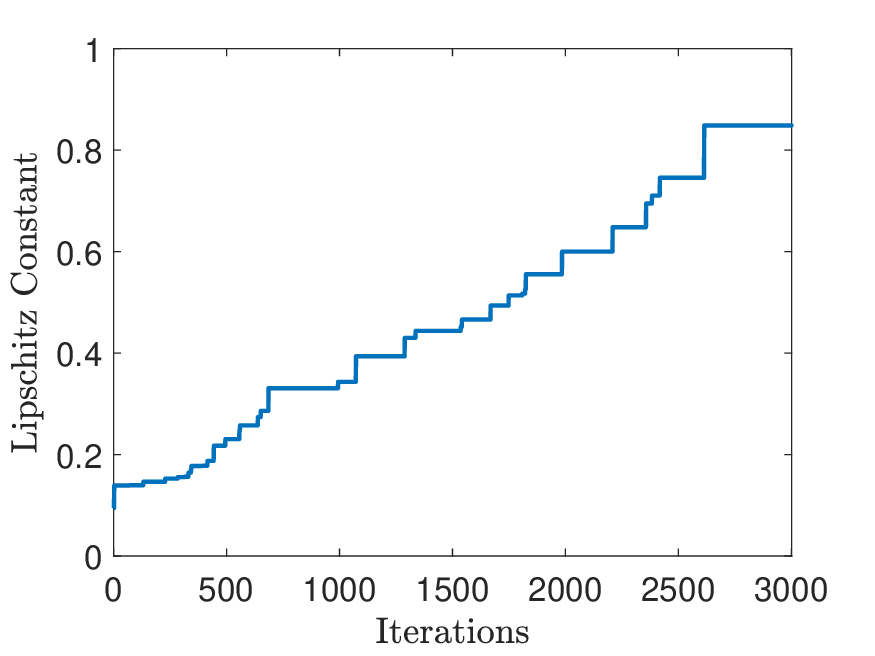}}
    \vspace{-5pt}
    \caption{Lipschitz constant}
    \label{fig:Lipschitz constant}
  \end{subfigure}
  \vspace{-2pt}
  \caption{Estimations on $G$ and $L'$}
\end{figure}
\subsection{Consensus and Convergence}
\par We compare our algorithm with that proposed in \cite{ye2024tradeoff}, noting that \cite{ye2024tradeoff} identifies PP-IOS as the best-performing algorithm. Therefore, we restrict our comparison to PP-IOS. Each normal agent's initial local model parameters are randomly initialized. The step size for the aggregation phase is set to $\beta_i(k) = 0.8$, and the total number of iterations is $K = 8000$. The local SGD step size is set to $\gamma = 0.01$. The noise level is chosen as $\sigma = 2.0$ for all methods. We adopt the Byzantine attack model introduced in \cite{guerraoui2018hidden}.

\par The experimental results are presented in Fig. \ref{fig:com_results}. Specifically, Fig. \ref{fig:accuracy_avg} shows the average test accuracy, while Fig. \ref{fig:avg_grad_norm} illustrates the squared norm of the gradient at the averaged model, i.e., $||\nabla f(\overline{x}(k))||^2$. Fig. \ref{fig:early_con_error} displays the consensus error during the first $0-15$ iterations, and Fig. \ref{fig:con_error} shows the consensus error over the entire training process.
 As shown in Fig. \ref{fig:accuracy_avg}, the proposed algorithm outperforms the method in \cite{ye2024tradeoff} in terms of accuracy. This improvement is attributed to the ability of our approach to accurately compute a point within the convex hull formed by the normal agents. Fig. \ref{fig:avg_grad_norm} indicates that the gradient norm fluctuates within a certain range rather than converging to zero, primarily due to the impact of the injected noise. Figs. \ref{fig:early_con_error} and \ref{fig:con_error} demonstrate that the model parameters of normal agents quickly converge to a low consensus error. However, due to the presence of noise and the heterogeneity of local datasets, exact consensus cannot be fully achieved.

\begin{figure}[htbp]
	\centering
	\begin{subfigure}{0.48\linewidth}
		\centering
		\includegraphics[width=1.0\linewidth]{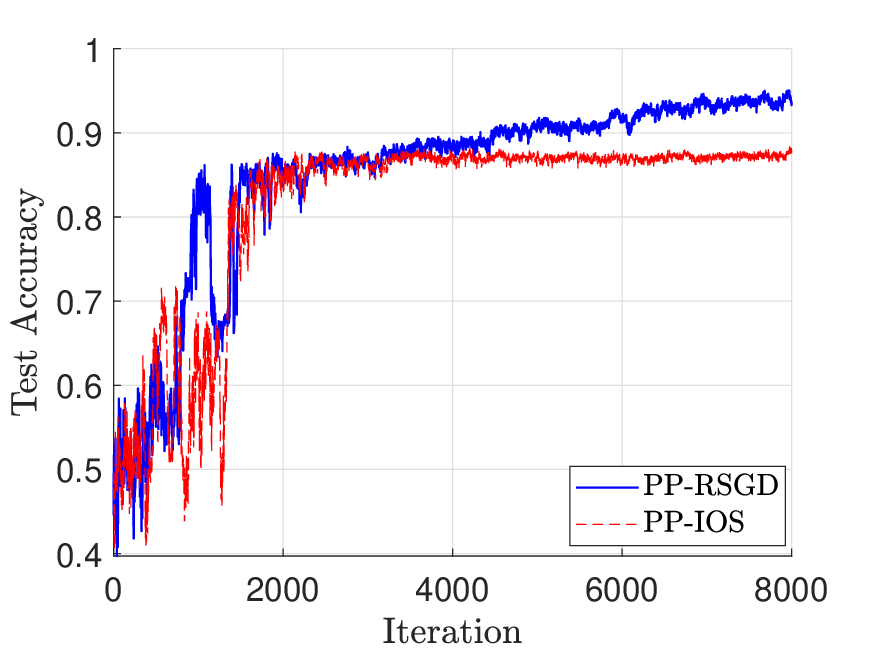}
		\caption{Test accuracy}
		\label{fig:accuracy_avg}
	\end{subfigure}
	\begin{subfigure}{0.48\linewidth}
		\centering
		\includegraphics[width=1.0\linewidth]{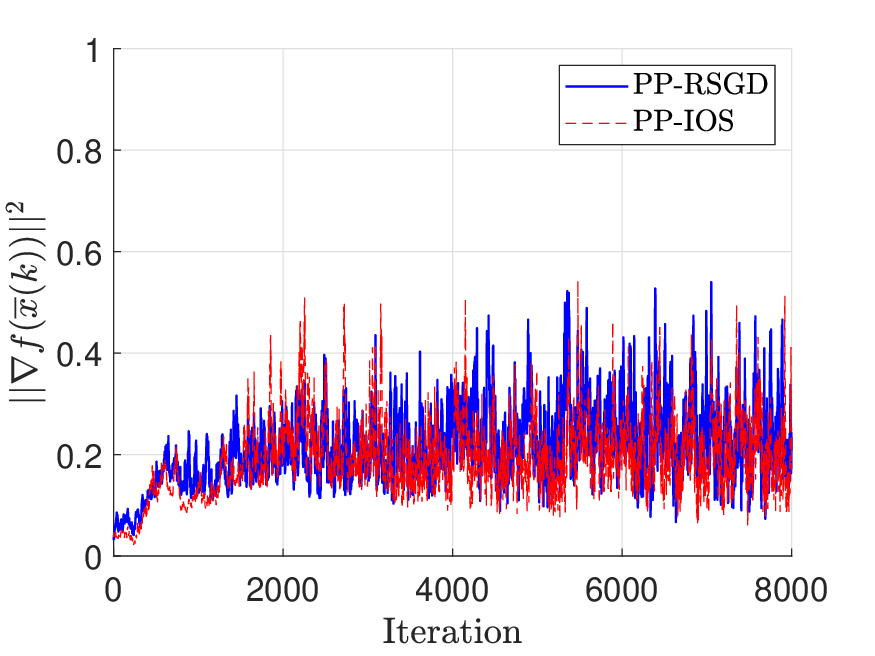}
		\caption{Norm $||\nabla f(\overline{x}(k))||^2$}
		\label{fig:avg_grad_norm}
	\end{subfigure}
	%\qquad
	\begin{subfigure}{0.48\linewidth}
		\centering
		\includegraphics[width=1.0\linewidth]{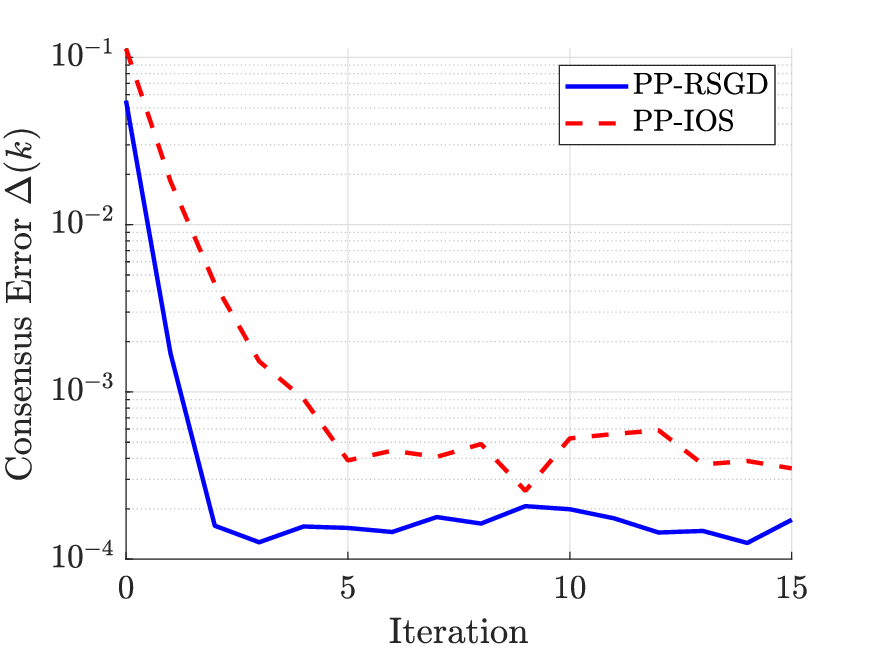}
		\caption{$\Delta(k)$ for $k=0-15$}
		\label{fig:early_con_error}
	\end{subfigure}
	\begin{subfigure}{0.48\linewidth}
		\centering
		\includegraphics[width=1.0\linewidth]{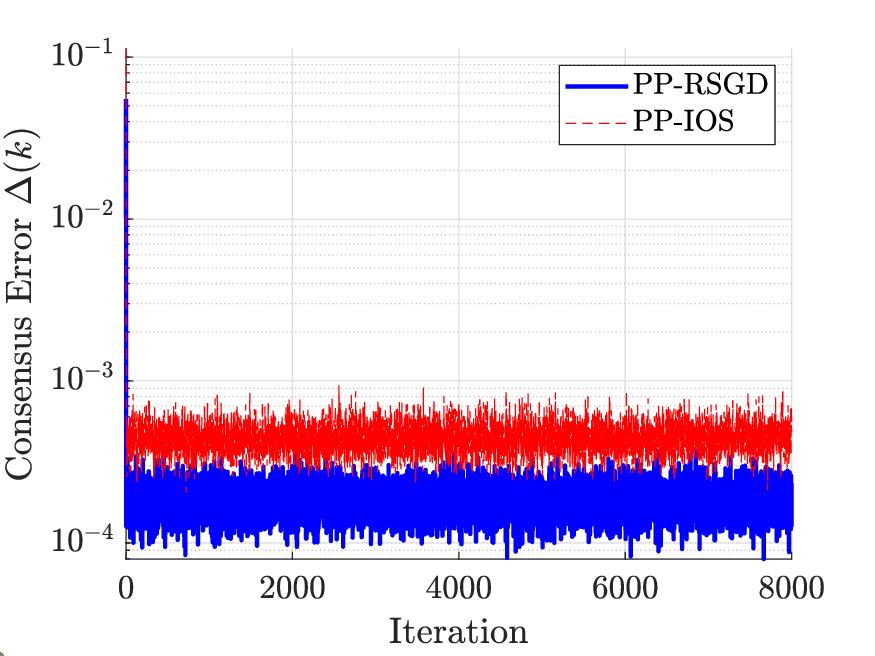}
		\caption{$\Delta(k)$ for all $k$}
		\label{fig:con_error}
	\end{subfigure}
\vspace{-1pt}
\caption{Comparison of learning results}
\label{fig:com_results}
\end{figure}
\subsection{Privacy Attack}
\par To evaluate the effectiveness of privacy preservation, we conduct gradient inversion attack experiments following the method proposed in \cite{zhu2019deep}. As for the labels, we use the algorithms proposed in 
 \cite{zhao2020idlg}. These methods are widely recognized as foundational approaches in the field of privacy attacks on deep learning, known as \emph{Deep Leakage from Gradients}. We conduct a privacy attack targeting first normal agent, using model parameters obtained via PyTorch's random initialization and a randomly selected data point. We compare the reconstruction performance under two settings: without noise and with Gaussian noise added ($\sigma = 2.0$). The vertical axis represents the distance between the estimated and true data points, while the horizontal axis indicates the number of GIA optimization iterations, indicating $\sigma=2.0$ is sufficient for the privacy preservation. 

\begin{figure}[H]
  \begin{subfigure}{0.48\linewidth}
    \centerline{\includegraphics[width=1.0\linewidth]{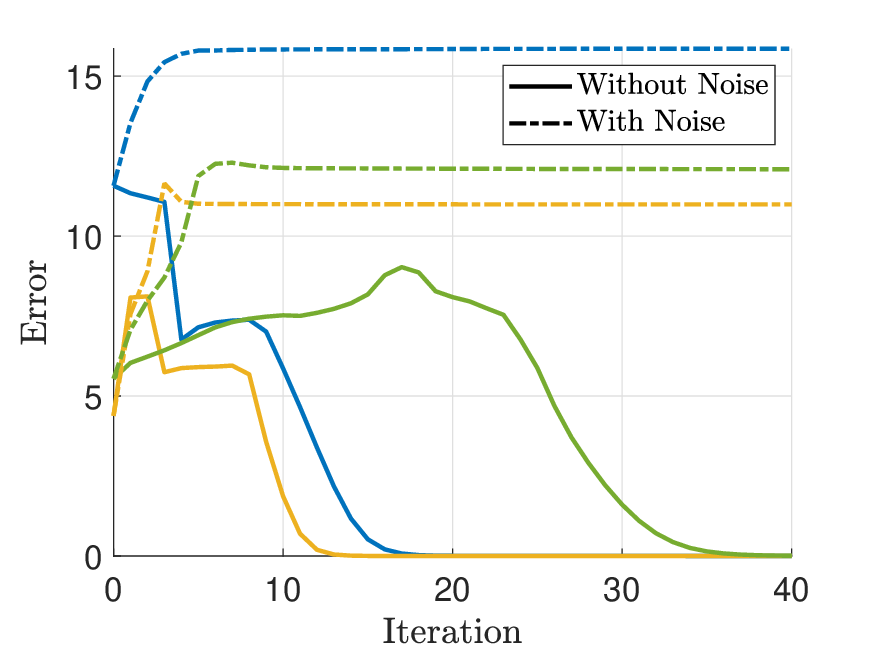}}
    \vspace{-5pt}
        \caption{Fast convergence}  
    \label{fig:gia_fast}  
  \end{subfigure}
  \hfill
  \begin{subfigure}{0.48\linewidth}
    \centerline{\includegraphics[width=1.0\linewidth]{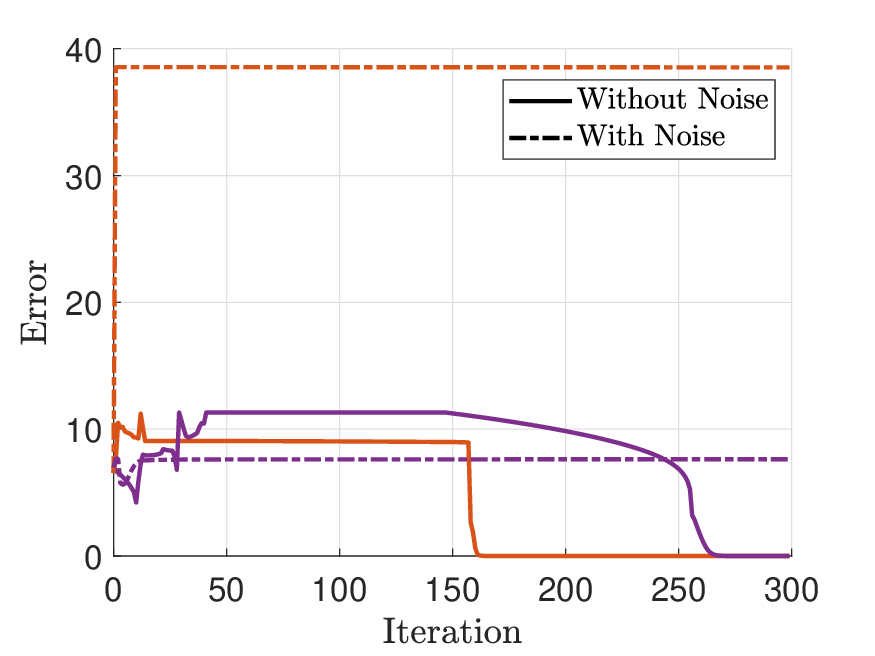}}
    \vspace{-5pt}
    \caption{Slow convergence}
    \label{fig:gia_slow}
  \end{subfigure}
  \vspace{-2pt}
  \caption{Gradient inversion attack}
    \label{fig:gia}
\end{figure}

\subsection{Privacy Metrics}
\par Finally, we validate the advantages of CGP on a better trade-off discussed in Section \ref{sec:tradeoff} through numerical simulations. By substituting $\delta=10^{-5}$ (a value commonly used in machine learning with DP) and other detailed parameters into DP, GP, and CGP, and treating $\sigma$ as a variable, we obtain the results shown in Fig. \ref{fig:privacy_metrics}.
\begin{figure}[H]
    \centering
    \includegraphics[width=0.7\linewidth]{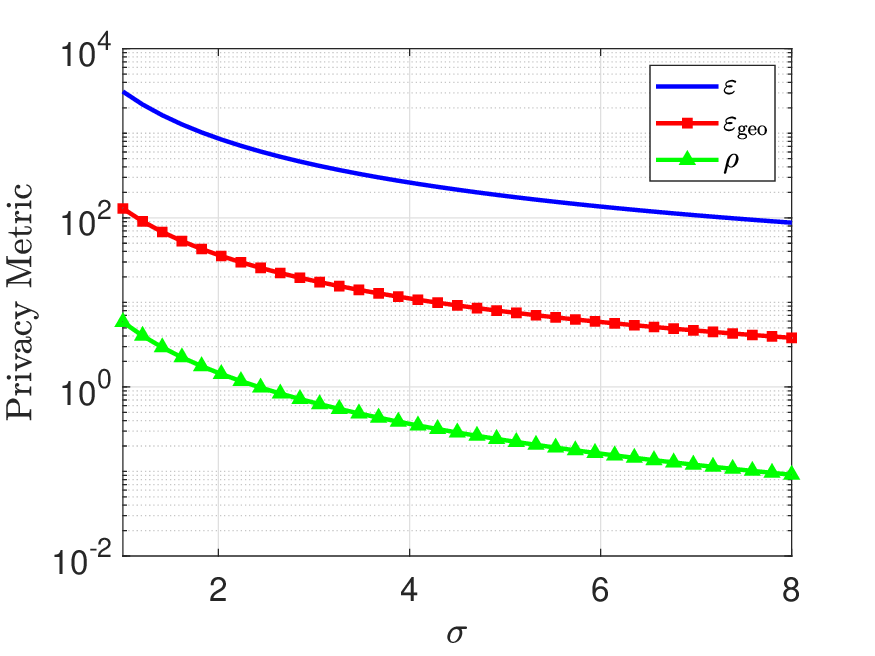}
    \caption{Different metrics of privacy}
    \label{fig:privacy_metrics}
\end{figure}
The results show that the DP parameter $\varepsilon$ is significantly larger than both $\varepsilon_{\mathrm{geo}}$ and $\rho$. Taking the simulation case with $\sigma=2.0$ as an example, the corresponding parameters are $\varepsilon=865$, $\varepsilon_{\mathrm{geo}}=36.2$, and $\rho=1.45$. Notably, $\sigma=2.0$ already provides reasonably strong privacy preservation. However, when using DP as the privacy metric, the worst-case scenario leads to an excessively large $\varepsilon$ value, making noise parameter selection impractical. Furthermore, we observe that $\rho$ is substantially smaller than $\varepsilon_{\mathrm{geo}}$. This difference primarily stems from GP's requirement to maintain a sufficiently small $\delta$ to ensure reliable privacy protection, whereas $\rho$-CGP is free from this limitation.
\par Next, we demonstrate another advantage of CGP: providing a more accurate privacy-accuracy trade-off. To illustrate this point, we plot the derivatives of $\varepsilon$, $\varepsilon_{\mathrm{geo}}$, and $\rho$ with respect to $-\sigma^2$. From Theorem \ref{th:convergence}, Since model accuracy is directly related to the noise parameter $\sigma^2$ (with smaller $\sigma^2$ values yielding higher model accuracy), the rate at which $\sigma^2$ decreases as model accuracy improves is positively correlated with the corresponding derivative.
\begin{figure}[H]
    \centering
    \includegraphics[width=0.7\linewidth]{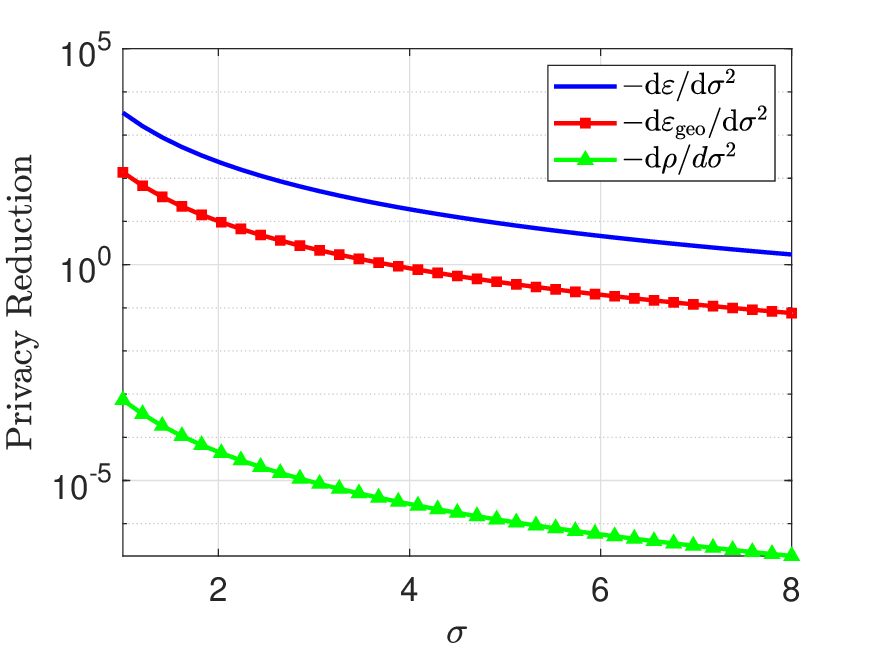}
    \caption{Privacy reduction}
    \label{fig:privacy_reduction}
\end{figure}
\par The figure clearly shows that $-\frac{\mathrm{d}\varepsilon}{\mathrm{d}\sigma^2} \gg -\frac{\mathrm{d}\varepsilon_\mathrm{geo}}{\mathrm{d}\sigma^2} \gg -\frac{\mathrm{d}\rho}{\mathrm{d}\sigma^2}$, demonstrating that $\rho$-CGP provides the most accurate privacy-accuracy trade-off among the compared notions.

\section{CONCLUSIONS}\label{sec:con}
This paper proposed a decentralized DML framework, named ImprovDML, that simultaneously addresses Byzantine resilience and privacy preservation. By employing RVC algorithms as the resilient aggregation rule and using CGP to quantify privacy strength, we demonstrated that our method achieves lower learning error and a more accurate trade-off between privacy and accuracy compared to existing approaches, while maintaining the same level of privacy and ensuring resilience against Byzantine faults. Future work will explore the use of temporally varying Gaussian noise with inter-dimensional correlations to further enhance training accuracy.

\bibliographystyle{plain} 
\bibliography{autosam}

\begin{thebibliography}{10}

\bibitem{abbas2022resilient}
Waseem Abbas, Mudassir Shabbir, Jiani Li, and Xenofon Koutsoukos.
\newblock Resilient distributed vector consensus using centerpoint.
\newblock {\em Automatica}, 136:110046, 2022.

\bibitem{baruch2019little}
Gilad Baruch, Moran Baruch, and Yoav Goldberg.
\newblock A little is enough: Circumventing defenses for distributed learning.
\newblock {\em Advances in Neural Information Processing Systems}, 32, 2019.

\bibitem{bassily2014private}
Raef Bassily, Adam Smith, and Abhradeep Thakurta.
\newblock Private empirical risk minimization: Efficient algorithms and tight error bounds.
\newblock In {\em 2014 IEEE 55th Annual Symposium on Foundations of Computer Science}, pages 464--473. IEEE, 2014.

\bibitem{beimel2014bounds}
Amos Beimel, Hai Brenner, Shiva~Prasad Kasiviswanathan, and Kobbi Nissim.
\newblock Bounds on the sample complexity for private learning and private data release.
\newblock {\em Machine Learning}, 94:401--437, 2014.

\bibitem{blanchard2017machine}
Peva Blanchard, El~Mahdi El~Mhamdi, Rachid Guerraoui, and Julien Stainer.
\newblock Machine learning with adversaries: Byzantine tolerant gradient descent.
\newblock {\em Advances in Neural Information Processing Systems}, 30, 2017.

\bibitem{bottou2018optimization}
L{\'e}on Bottou, Frank~E Curtis, and Jorge Nocedal.
\newblock Optimization methods for large-scale machine learning.
\newblock {\em SIAM review}, 60(2):223--311, 2018.

\bibitem{chan2004optimal}
Timothy~M Chan.
\newblock An optimal randomized algorithm for maximum tukey depth.
\newblock In {\em Proceedings of the Fifteenth Annual ACM-SIAM Symposium on Discrete Algorithms}, volume~4, pages 430--436, 2004.

\bibitem{chen2017distributed}
Yudong Chen, Lili Su, and Jiaming Xu.
\newblock Distributed statistical machine learning in adversarial settings: Byzantine gradient descent.
\newblock {\em Proceedings of the ACM on Measurement and Analysis of Computing Systems}, 1(2):1--25, 2017.

\bibitem{cyffers2022privacy}
Edwige Cyffers and Aur{\'e}lien Bellet.
\newblock Privacy amplification by decentralization.
\newblock In {\em International Conference on Artificial Intelligence and Statistics}, pages 5334--5353. PMLR, 2022.

\bibitem{cyffers2024differentially}
Edwige Cyffers, Aur{\'e}lien Bellet, and Jalaj Upadhyay.
\newblock Differentially private decentralized learning with random walks.
\newblock In {\em International Conference on Machine Learning}, pages 9762--9783. PMLR, 2024.

\bibitem{cyffers2022muffliato}
Edwige Cyffers, Mathieu Even, Aur{\'e}lien Bellet, and Laurent Massouli{\'e}.
\newblock Muffliato: Peer-to-peer privacy amplification for decentralized optimization and averaging.
\newblock {\em Advances in Neural Information Processing Systems}, 35:15889--15902, 2022.

\bibitem{danzer1963helly}
Ludwig Danzer, Branko Gr{\"u}nbaum, and Victor Klee.
\newblock Helly's theorem and its relatives.
\newblock In {\em Convexity: Proceedings of the Seventh Symposium in Pure Mathematics of the American Mathematical Society}, volume~7, page 101. American Mathematical Soc., 1963.

\bibitem{dean2012large}
Jeffrey Dean, Greg Corrado, Rajat Monga, Kai Chen, Matthieu Devin, Mark Mao, Marc'aurelio Ranzato, Andrew Senior, Paul Tucker, Ke~Yang, et~al.
\newblock Large scale distributed deep networks.
\newblock {\em Advances in Neural Information Processing Systems}, 25, 2012.

\bibitem{dwork2014algorithmic}
Cynthia Dwork, Aaron Roth, et~al.
\newblock The algorithmic foundations of differential privacy.
\newblock {\em Foundations and Trends{\textregistered} in Theoretical Computer Science}, 9(3--4):211--407, 2014.

\bibitem{fang2022bridge}
Cheng Fang, Zhixiong Yang, and Waheed~U Bajwa.
\newblock Bridge: Byzantine-resilient decentralized gradient descent.
\newblock {\em IEEE Transactions on Signal and Information Processing over Networks}, 8:610--626, 2022.

\bibitem{fang2020local}
Minghong Fang, Xiaoyu Cao, Jinyuan Jia, and Neil Gong.
\newblock Local model poisoning attacks to byzantine-robust federated learning.
\newblock In {\em 29th USENIX Security Symposium}, pages 1605--1622, 2020.

\bibitem{fang2024byzantine}
Minghong Fang, Zifan Zhang, Hairi, Prashant Khanduri, Jia Liu, Songtao Lu, Yuchen Liu, and Neil Gong.
\newblock Byzantine-robust decentralized federated learning.
\newblock In {\em Proceedings of the 2024 on ACM SIGSAC Conference on Computer and Communications Security}, pages 2874--2888, 2024.

\bibitem{fredrikson2015model}
Matt Fredrikson, Somesh Jha, and Thomas Ristenpart.
\newblock Model inversion attacks that exploit confidence information and basic countermeasures.
\newblock In {\em Proceedings of the 22nd ACM SIGSAC Conference on Computer and Communications Security}, pages 1322--1333, 2015.

\bibitem{ghavamipour2024privacy}
Ali~Reza Ghavamipour, Benjamin Zi~Hao Zhao, Oguzhan Ersoy, and Fatih Turkmen.
\newblock Privacy-preserving aggregation for decentralized learning with byzantine-robustness.
\newblock {\em arXiv preprint arXiv:2404.17970}, 2024.

\bibitem{guerraoui2018hidden}
Rachid Guerraoui, S{\'e}bastien Rouault, et~al.
\newblock The hidden vulnerability of distributed learning in byzantium.
\newblock In {\em International Conference on Machine Learning}, pages 3521--3530. PMLR, 2018.

\bibitem{hatamizadeh2022gradvit}
Ali Hatamizadeh, Hongxu Yin, Holger~R Roth, Wenqi Li, Jan Kautz, Daguang Xu, and Pavlo Molchanov.
\newblock Gradvit: Gradient inversion of vision transformers.
\newblock In {\em Proceedings of the IEEE/CVF Conference on Computer Vision and Pattern Recognition}, pages 10021--10030, 2022.

\bibitem{he2022byzantine}
Lie He, Sai~Praneeth Karimireddy, and Martin Jaggi.
\newblock Byzantine-robust decentralized learning via clippedgossip.
\newblock {\em arXiv preprint arXiv:2202.01545}, 2022.

\bibitem{kasiviswanathan2011can}
Shiva~Prasad Kasiviswanathan, Homin~K Lee, Kobbi Nissim, Sofya Raskhodnikova, and Adam Smith.
\newblock What can we learn privately?
\newblock {\em SIAM Journal on Computing}, 40(3):793--826, 2011.

\bibitem{koloskova2023revisiting}
Anastasia Koloskova, Hadrien Hendrikx, and Sebastian~U Stich.
\newblock Revisiting gradient clipping: Stochastic bias and tight convergence guarantees.
\newblock In {\em International Conference on Machine Learning}, pages 17343--17363. PMLR, 2023.

\bibitem{koloskova2020unified}
Anastasia Koloskova, Nicolas Loizou, Sadra Boreiri, Martin Jaggi, and Sebastian Stich.
\newblock A unified theory of decentralized sgd with changing topology and local updates.
\newblock In {\em International Conference on Machine Learning}, pages 5381--5393. PMLR, 2020.

\bibitem{leblanc2013resilient}
Heath~J LeBlanc, Haotian Zhang, Xenofon Koutsoukos, and Shreyas Sundaram.
\newblock Resilient asymptotic consensus in robust networks.
\newblock {\em IEEE Journal on Selected Areas in Communications}, 31(4):766--781, 2013.

\bibitem{li2022byzantine}
Jiani Li, Waseem Abbas, Mudassir Shabbir, and Xenofon Koutsoukos.
\newblock Byzantine resilient distributed learning in multirobot systems.
\newblock {\em IEEE Transactions on Robotics}, 38(6):3550--3563, 2022.

\bibitem{lian2017can}
Xiangru Lian, Ce~Zhang, Huan Zhang, Cho-Jui Hsieh, Wei Zhang, and Ji~Liu.
\newblock Can decentralized algorithms outperform centralized algorithms? a case study for decentralized parallel stochastic gradient descent.
\newblock {\em Advances in Neural Information Processing Systems}, 30, 2017.

\bibitem{liang2023concentrated}
Yuting Liang and Ke~Yi.
\newblock Concentrated geo-privacy.
\newblock In {\em Proceedings of the 2023 ACM SIGSAC Conference on Computer and Communications Security}, pages 1934--1948, 2023.

\bibitem{lu2023privacy}
Yang Lu, Zhengxin Yu, and Neeraj Suri.
\newblock Privacy-preserving decentralized federated learning over time-varying communication graph.
\newblock {\em ACM Transactions on Privacy and Security}, 26(3):1--39, 2023.

\bibitem{mendes2013multidimensional}
Hammurabi Mendes and Maurice Herlihy.
\newblock Multidimensional approximate agreement in byzantine asynchronous systems.
\newblock In {\em Proceedings of the Forty-Fifth Annual ACM Symposium on Theory of Computing}, pages 391--400, 2013.

\bibitem{mironov2017renyi}
Ilya Mironov.
\newblock R{\'e}nyi differential privacy.
\newblock In {\em 2017 IEEE 30th Computer Security Foundations Symposium}, pages 263--275. IEEE, 2017.

\bibitem{park2017fault}
Hyongju Park and Seth~A Hutchinson.
\newblock Fault-tolerant rendezvous of multirobot systems.
\newblock {\em IEEE Transactions on Robotics}, 33(3):565--582, 2017.

\bibitem{renyi1961measures}
Alfr{\'e}d R{\'e}nyi.
\newblock On measures of entropy and information.
\newblock In {\em Proceedings of the Fourth Berkeley Symposium on Mathematical Statistics and Probability, Volume 1: Contributions to the Theory of Statistics}, volume~4, pages 547--562. University of California Press, 1961.

\bibitem{so2020scalable}
Jinhyun So, Basak Guler, and Salman Avestimehr.
\newblock A scalable approach for privacy-preserving collaborative machine learning.
\newblock {\em Advances in Neural Information Processing Systems}, 33:8054--8066, 2020.

\bibitem{srivastava2015training}
Rupesh~K Srivastava, Klaus Greff, and J{\"u}rgen Schmidhuber.
\newblock Training very deep networks.
\newblock {\em Advances in Neural Information Processing Systems}, 28, 2015.

\bibitem{sun2022decentralized}
Tao Sun, Dongsheng Li, and Bao Wang.
\newblock Decentralized federated averaging.
\newblock {\em IEEE Transactions on Pattern Analysis and Machine Intelligence}, 45(4):4289--4301, 2022.

\bibitem{vaidya2014iterative}
Nitin~H Vaidya.
\newblock Iterative byzantine vector consensus in incomplete graphs.
\newblock In {\em Distributed Computing and Networking: 15th International Conference, ICDCN 2014, Coimbatore, India, January 4-7, 2014. Proceedings 15}, pages 14--28. Springer, 2014.

\bibitem{wang2018resilient}
Xuan Wang, Shaoshuai Mou, and Shreyas Sundaram.
\newblock A resilient convex combination for consensus-based distributed algorithms.
\newblock {\em arXiv preprint arXiv:1806.10271}, 2018.

\bibitem{wang2023tailoring}
Yongqiang Wang and Angelia Nedi{\'c}.
\newblock Tailoring gradient methods for differentially private distributed optimization.
\newblock {\em IEEE Transactions on Automatic Control}, 69(2):872--887, 2023.

\bibitem{wu2023byzantine}
Zhaoxian Wu, Tianyi Chen, and Qing Ling.
\newblock Byzantine-resilient decentralized stochastic optimization with robust aggregation rules.
\newblock {\em IEEE Transactions on Signal Processing}, 2023.

\bibitem{xia2019faba}
Qi~Xia, Zeyi Tao, Zijiang Hao, and Qun Li.
\newblock Faba: an algorithm for fast aggregation against byzantine attacks in distributed neural networks.
\newblock In {\em International Joint Conference on Artificial Intelligence}, page 4824–4830, 2019.

\bibitem{xu2021dp}
Jie Xu, Wei Zhang, and Fei Wang.
\newblock $\text{A}(\text{DP})^2\text{SGD}$: Asynchronous decentralized parallel stochastic gradient descent with differential privacy.
\newblock {\em IEEE Transactions on Pattern Analysis and Machine Intelligence}, 44(11):8036--8047, 2021.

\bibitem{yan2022resilient}
Jiaqi Yan, Xiuxian Li, Yilin Mo, and Changyun Wen.
\newblock Resilient multi-dimensional consensus in adversarial environment.
\newblock {\em Automatica}, 145:110530, 2022.

\bibitem{yang2024byzantine}
Caiyi Yang and Javad Ghaderi.
\newblock Byzantine-robust decentralized learning via remove-then-clip aggregation.
\newblock In {\em Proceedings of the AAAI Conference on Artificial Intelligence}, volume~38, pages 21735--21743, 2024.

\bibitem{yang2019byrdie}
Zhixiong Yang and Waheed~U Bajwa.
\newblock Byrdie: Byzantine-resilient distributed coordinate descent for decentralized learning.
\newblock {\em IEEE Transactions on Signal and Information Processing over Networks}, 5(4):611--627, 2019.

\bibitem{ye2024tradeoff}
Haoxiang Ye, Heng Zhu, and Qing Ling.
\newblock On the tradeoff between privacy preservation and byzantine-robustness in decentralized learning.
\newblock In {\em ICASSP 2024-2024 IEEE International Conference on Acoustics, Speech and Signal Processing}, pages 9336--9340. IEEE, 2024.

\bibitem{ye2011interior}
Yinyu Ye.
\newblock {\em Interior point algorithms: theory and analysis}.
\newblock John Wiley \& Sons, 2011.

\bibitem{yin2018byzantine}
Dong Yin, Yudong Chen, Ramchandran Kannan, and Peter Bartlett.
\newblock Byzantine-robust distributed learning: Towards optimal statistical rates.
\newblock In {\em International Conference on Machine Learning}, pages 5650--5659. Pmlr, 2018.

\bibitem{wang2019subsampled}
{Yu-Xiang Wang}, {Borja Balle}, and {Shiva Kasiviswanathan}.
\newblock Subsampled r{\'e}nyi differential privacy and analytical moments accountant.
\newblock In {\em The 22nd International Conference on Artificial Intelligence and Statistics}, pages 1226--1235. PMLR, 2019.

\bibitem{zhang2018admm}
Chunlei Zhang, Muaz Ahmad, and Yongqiang Wang.
\newblock Admm based privacy-preserving decentralized optimization.
\newblock {\em IEEE Transactions on Information Forensics and Security}, 14(3):565--580, 2018.

\bibitem{zhang2020secret}
Yuheng Zhang, Ruoxi Jia, Hengzhi Pei, Wenxiao Wang, Bo~Li, and Dawn Song.
\newblock The secret revealer: Generative model-inversion attacks against deep neural networks.
\newblock In {\em Proceedings of the IEEE/CVF Conference on Computer Vision and Pattern Recognition}, pages 253--261, 2020.

\bibitem{zhao2020idlg}
Bo~Zhao, Konda~Reddy Mopuri, and Hakan Bilen.
\newblock idlg: Improved deep leakage from gradients.
\newblock {\em arXiv preprint arXiv:2001.02610}, 2020.

\bibitem{zhao2022pvd}
Jiaqi Zhao, Hui Zhu, Fengwei Wang, Rongxing Lu, Zhe Liu, and Hui Li.
\newblock Pvd-fl: A privacy-preserving and verifiable decentralized federated learning framework.
\newblock {\em IEEE Transactions on Information Forensics and Security}, 17:2059--2073, 2022.

\bibitem{zhu2019deep}
Ligeng Zhu, Zhijian Liu, and Song Han.
\newblock Deep leakage from gradients.
\newblock {\em Advances in Neural Information Processing Systems}, 32, 2019.

\end{thebibliography}

\appendix

\end{document}